\documentclass[11pt,reqno]{amsart}
\usepackage{graphicx,enumerate,amssymb,bbold}
\usepackage[T1]{fontenc}
\usepackage[applemac]{inputenc}
\usepackage[notrig]{physics}
\usepackage[font=footnotesize,labelfont=small]{subcaption}
\usepackage[ruled,vlined]{algorithm2e}
\usepackage[left=2.5cm,right=2.5cm,top=2.5cm,bottom=2.5cm,includeheadfoot]{geometry} 

\usepackage[colorlinks=true, pdfstartview=FitV, linkcolor=blue, 
            citecolor=blue, urlcolor=blue]{hyperref}

\usepackage{caption}
\usepackage{subcaption}
\usepackage{todonotes}
\usepackage{booktabs}

\numberwithin{equation}{section}
\numberwithin{figure}{section}
\theoremstyle{plain}
\newtheorem{theorem}{Theorem}[section]
\newtheorem{lemma}[theorem]{Lemma}
\newtheorem{proposition}[theorem]{Proposition}

\theoremstyle{definition}

\newtheorem{example}[theorem]{Example}

\newcommand{\proposaldist}{\varrho}

\newcommand{\bitem}{\begin{itemize}}
\newcommand{\eitem}{\end{itemize}}
\newcommand{\mc}[1]{\mathcal{#1}}

\newcommand{\N}{\mathbb{N}}
\newcommand{\R}{\mathbb{R}}

\newcommand{\EE}{\mathbb{E}}

\newcommand{\PP}{\mathbb{P}}

\newcommand{\bpm}{\begin{pmatrix}}
\newcommand{\epm}{\end{pmatrix}}
\newcommand{\bvm}{\begin{vmatrix}}
\newcommand{\evm}{\end{vmatrix}}
\newcommand{\bsm}{\left(\begin{smallmatrix}}
\newcommand{\esm}{\end{smallmatrix}\right)}
\newcommand{\T}{\top}

\newcommand{\ol}[1]{\overline{#1}}
\newcommand{\wh}[1]{\widehat{#1}}
\newcommand{\wt}[1]{\widetilde{#1}}
\newcommand{\la}{\langle}
\newcommand{\ra}{\rangle}

\newcommand{\mrm}[1]{\mathrm{#1}}

\newcommand{\eins}{\mathbb{1}}

\DeclareMathSymbol{\mydiv}{\mathbin}{symbols}{"04}

\DeclareMathOperator{\mtrace}{tr}
\DeclareMathOperator{\mimg}{img}
\DeclareMathOperator{\proj}{proj}
\DeclareMathOperator{\Diag}{Diag}

\DeclareMathOperator{\Exp}{Exp}

\DeclareMathOperator{\KL}{KL}

\makeatletter
\def\widebreve{\mathpalette\wide@breve}
\def\wide@breve#1#2{\sbox\z@{$#1#2$}%
     \mathop{\vbox{\m@th\ialign{##\crcr
\kern0.08em\brevefill#1{0.8\wd\z@}\crcr\noalign{\nointerlineskip}%
                    $\hss#1#2\hss$\crcr}}}\limits}
\def\brevefill#1#2{$\m@th\sbox\tw@{$#1($}%
  \hss\resizebox{#2}{\wd\tw@}{\rotatebox[origin=c]{90}{\upshape(}}\hss$}
\makeatletter

\title[Generative Assignment Flows]{Generative Assignment Flows for Representing \\ and Learning Joint Distributions of Discrete Data}
\author[B.~Boll, D.~Gonzalez-Alvarado, S.~Petra, C.~Schn\"{o}rr]{Bastian Boll, Daniel Gonzalez-Alvarado, Stefania Petra, Christoph Schn\"{o}rr}
\address[B.~Boll, D.~Gonzalez-Alvarado]{Institute for Mathematics, Image and Pattern Analysis Group, Heidelberg University, Germany \newline
Corresponding author: Daniel Gonzalez-Alvarado} 
\email{daniel.gonzalez@iwr.uni-heidelberg.de}
\urladdr{\url{https://ipa.math.uni-heidelberg.de}}
\address[S.~Petra]{Mathematical Imaging Group,  Department of Mathematics \& Centre for Advanced Analytics and Predictive Sciences
(CAAPS), University of Augsburg, Universitätsstr. 14, 86159 Augsburg, Germany} 
\email{Stefania.Petra@uni-a.de}
\urladdr{\url{https://www.uni-augsburg.de/de/fakultaet/mntf/math/prof/mig/}}
\address[ C.~Schn\"{o}rr]{Institute for Mathematics \& Research Station Geometry and Dynamics,  Image and Pattern Analysis Group, Heidelberg University, Germany} 
\email{schnoerr@math.uni-heidelberg.de}
\urladdr{\url{https://ipa.math.uni-heidelberg.de}}
\date{} 

\thanks{This work was funded by the Deutsche Forschungsgemeinschaft (DFG), grant SCHN 457/17-1, within the priority programme SPP 2298: Theoretical Foundations of Deep Learning. This work was funded by the Deutsche Forschungsgemeinschaft (DFG) under Germany’s Excellence Strategy EXC-2181/1-390900948 (the Heidelberg STRUCTURES Excellence Cluster).}



\begin{document}

\begin{abstract}
We introduce a novel generative model for the representation of joint probability distributions of a possibly large number of discrete random variables. The approach uses measure transport by randomized assignment flows on the statistical submanifold of factorizing distributions, which enables to represent and sample efficiently from any target distribution and to assess the likelihood of unseen data points. The complexity of the target distribution only depends on the parametrization of the affinity function of the dynamical assignment flow system.
Our model can be trained in a simulation-free manner by conditional Riemannian flow matching, using the training data encoded as geodesics on the assignment manifold in closed-form, with respect to the e-connection of information geometry.
Numerical experiments devoted to distributions of structured image labelings demonstrate the applicability to large-scale problems, which may include discrete distributions in other application areas. Performance measures show that our approach scales better with the increasing number of classes than recent related work.
\end{abstract}

\maketitle
\tableofcontents

\section{Introduction}\label{sec:Introduction}

\subsection{Overview, Motivation}
\textit{Generative models} in machine learning define an active area of research \cite{Kobyzev:2019aa,Papamakarios:2021vu,Ruthotto:2021tv}. Corresponding research objectives include 
\begin{enumerate}[(i)]
\item
the representation of complex probability distributions, 
\item 
efficient sampling from such distributions, and 
\item
computing the likelihoods of unseen data points. 
\end{enumerate}
The target probability distribution is typically not given, except for a finite sample set (empirical measure). The modeling task concerns the generation of the target distribution by transporting a simple reference measure, typically the multivariate standard normal distribution, using a corresponding pushforward mapping. This mapping is realized by a network with trainable parameters that are optimized by maximizing the likelihood of the given data or a corresponding surrogate objective which is more convenient regarding numerical optimization. This class of approaches are called \textit{normalizing flows} in the literature.

Discrete joint probability distributions abound in applications, yet have received less attention in the literature on generative models. The recent survey paper \cite{Kobyzev:2019aa} concludes with a short paragraph devoted to discrete distributions and the assessment that ``the generalization of normalizing flows to discrete distributions remains an open problem''. Likewise, the survey paper \cite{Papamakarios:2021vu} briefly discusses generative models of discrete distributions in \cite[Section 5.3]{Papamakarios:2021vu}. The authors state that ``compared to flows on $\R^{D}$, discrete flows have notable theoretical limitations''. The survey paper \cite{Ruthotto:2021tv} does not mention at all generative models of discrete distributions.

This paper introduces a novel generative approach for the significant subclass of \textit{discrete} (\textit{categorial}) probability distributions of $n$ random variables $y_{i}$ taking values in a finite set $\{1,2,\dotsc,c\}$,
\begin{equation}
y = (y_{1},\dotsc,y_{n})^{\T}\in [c]^{n},\qquad
y_{i}\in[c]:=\{1,2,\dotsc,c\},\qquad i\in[n],\qquad c,n \in\N.
\end{equation}
A corresponding distribution $p$ is a look-up table which specifies for any realization $\alpha$ of the discrete random vector $y$ the probability
\begin{equation}\label{eq:def-p-general}
p(\alpha) = p(\alpha_{1},\dotsc,\alpha_{n})
:= \Pr(y=\alpha)
= \Pr(y_{1}=\alpha_{1} \;\wedge\dotsb\wedge \;y_{n}=\alpha_{n}),\qquad \alpha\in[c]^{n}.
\end{equation}
Any such look-up table is a nonnegative tensor with the combinatorially large number  
\begin{equation}\label{eq:def-N-cn}
N := c^{n}
\end{equation}
of entries $p(\alpha),\,\alpha\in[c]^{n}$. Furthermore, since $p(\alpha)\geq 0,\,\forall \alpha$, and $\sum_{\alpha\in[c]^{n}}p(\alpha)=1$, any distribution $p$ also corresponds to a point $p\in\Delta_{N}$ of the probability simplex
\begin{equation}\label{def:Delta-N}
\Delta_{N}:=\{p\in\R_{\geq 0}^{N}\colon \la\eins_{N},p\ra=1\},\qquad
p = (p_{\alpha})_{\alpha\in[c]^{n}},\quad 
p_{\alpha} := p(\alpha),
\qquad
(\text{meta-simplex})
\end{equation}
where $\eins_{N}:=(1,1,\dotsc,1)^{\T}\in\R^{N}$.

Thus, we denote with $p$ discrete joint probability distributions using any of the equivalent representations 
\begin{itemize}
\item as functions $p\colon [c]^{n}\to[0,1]$, cf.~Eq.~\eqref{eq:def-p-general};
\item as nonnegative tensors with $c^{n}$ components $p(\alpha_{1},\dotsc,\alpha_{n})$;
\item as discrete probability vectors $p\in\Delta_{N}$ with $N=c^{n}$ components $p_{\alpha}$, where each component specifies the probability $p_{\alpha}=p(\alpha)=\Pr(y=\alpha)$, cf.~Eq.~\eqref{def:Delta-N}. 
\end{itemize}
In particular, the $N$ vertices (extreme points)  
\begin{equation}\label{eq:def-e-alpha}
e_{\alpha} \in \{0,1\}^{N}
\end{equation}
of $\Delta_{N}$ 
are the unit vectors which encode the discrete Dirac measures $\delta_{\alpha}$ concentrated on the realizations $\alpha\in [c]^{n}$.

\begin{figure}
    \centerline{
    \includegraphics[width=0.35\textwidth]{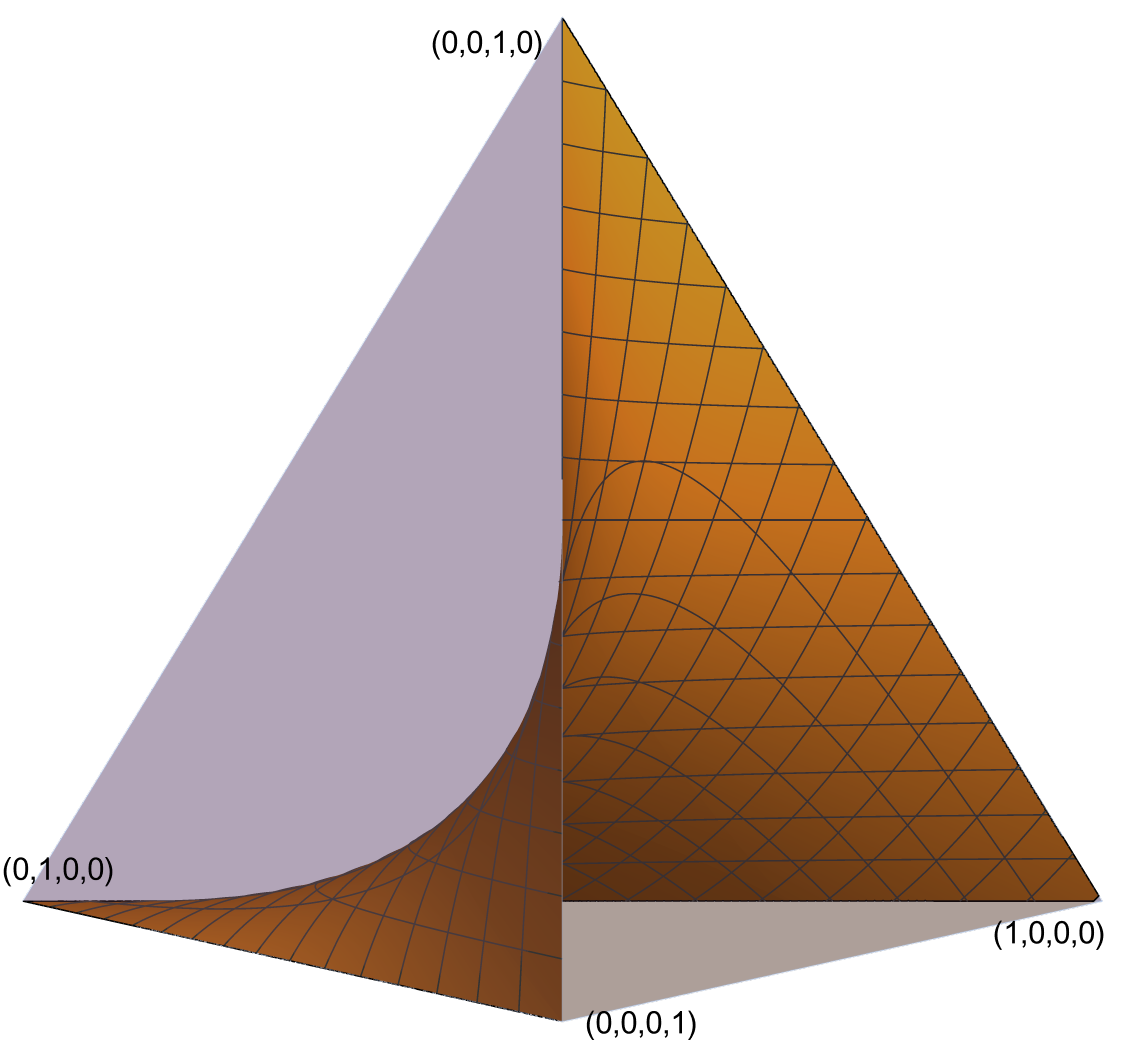}
    \hspace{0.05\textwidth}
    \includegraphics[width=0.35\textwidth]{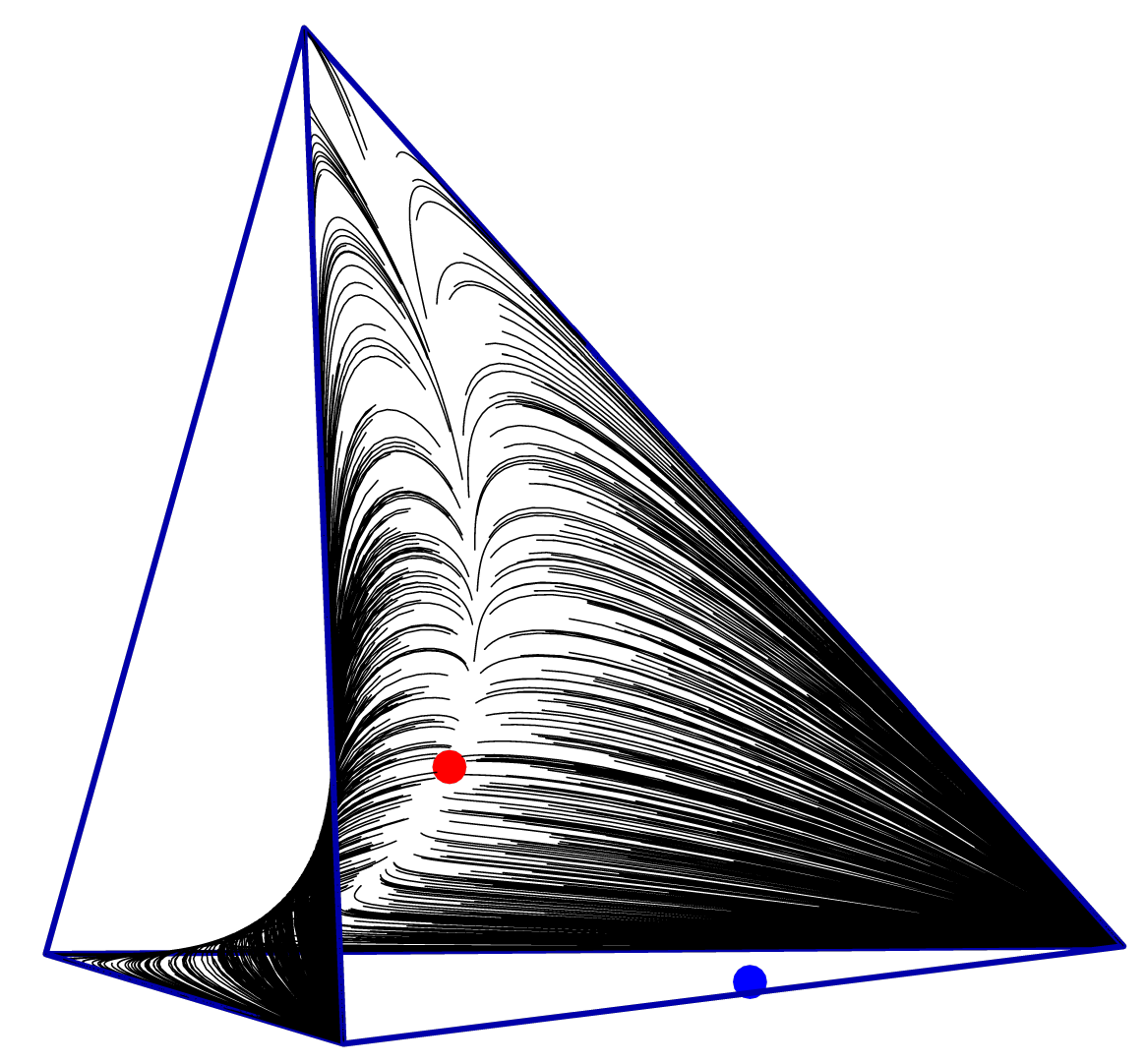}
    }
\centerline{
\parbox{0.35\textwidth}{\centering (a)}
\hspace{0.05\textwidth}
\parbox{0.35\textwidth}{\centering (b)}
}
    \caption{\textbf{(a)} The simplex $\Delta_{N}$ \eqref{def:Delta-N}, for $N=4$, depicted in local coordinates, and the submanifold of factorizing discrete distributions which connects all extreme points of $\Delta_{4}$. \textbf{(b)} Visualization of 1000 samples from the target distribution $p(\alpha_{1},\alpha_{2})$ given by \eqref{eq:py1y2}, corresponding to the blue point $p\in\Delta_{4}$. Each sample corresponds to an integral curve of a flow which evolves on the  submanifold and can be computed efficiently by geometric integration. The parametrized vector field of the dynamical system which generates the flow has been trainined by matching the flow to geodesics on the submanifold which encode given training data. As a result, each component $p_{\alpha}$ of the target distribution corresponds to the relative frequency of integral curves converging to the vertex $e_{\alpha}$, such that the entire distribution $p$ is represented by the convex combination $\sum_{\alpha} p_{\alpha} e_{\alpha} = p$. In this way, the flow realizes the pushforward of a simple reference distribution, centered at $0$ in the tangent space at the barycenter (red point), to the discrete target distribution $p$. Figure \ref{fig:approach} (p.~\pageref{fig:approach}) provides a more detailled illustration of the approach.}
    
    \label{fig:Wright-AF}
\end{figure}

Figure \ref{fig:Wright-AF} illustrates the approach for the toy distribution of two binary variables, i.e.~$c=2$ and $N=2^{2}=4$,
\begin{equation}\label{eq:py1y2}
p(\alpha_{1},\alpha_{2})\colon\qquad
\text{
\begin{tabular}{|c||c|c|}\hline
$\alpha_{1}/\alpha_{2}$ & 0 & 1 \\ \hline
0 & 0.45 & 0.05 \\ \hline
1 & 0.05 & 0.45 \\ \hline
\end{tabular}
}
\end{equation}
The simplex $\Delta_{4}\subset\R^{4}$ \eqref{def:Delta-N}  is visualized in $\R^{3}$ in local coordinates as tetrahedron (Figure \ref{fig:Wright-AF}(a); see Example \ref{ex:WrightManifold} (p.~\pageref{ex:WrightManifold}) for details). The generative model only uses the submanifold of \textit{factorizing} discrete distributions which ensures computational efficiency of both training and sampling. Figure \ref{fig:Wright-AF}(a) shows that this submanifold connects all extreme points of $\Delta_{4}$. 

Figure \ref{fig:Wright-AF}(b) illustrates how sampling from $p$ is accomplished after training, by computing on the submanifold the integral curves of a generating flow which emanates from initial random points, such that each curve converges to a vertex of the simplex which represents a realization $\alpha\sim p(\alpha)$ by \eqref{eq:def-e-alpha}. In this way, a simple reference distribution is pushed forward to $p$. Figure \ref{fig:approach} (p.~\pageref{fig:approach}) gives a more detailed account of the ingredients of our approach.

Training concerns the parameters of the vector field of the dynamical system, which generates the aforementioned flow on the submanifold. This is achieved by matching the flow to closed-form geodesics on the submanifold which encode given training data. This \textit{flow matching approach} has been recently proposed by \cite{Lipman:2023aa,Chen:2023}. Our paper elaborates this approach for \textit{discrete} joint probability distributions using the geometric approach outlined above.


%
\subsection{Related Work}\label{sec:Related-Work}

The central theme of our paper are large joint distributions of discrete random variables which has been a core topic in \textit{multivariate} and \textit{algebraic statistics}, with numerous applications in terms of discrete graphical models in various fields. In addition, our paper contributes to research on \textit{generative models} in \textit{machine learning}. Related work is accordingly reported in Sections \ref{sec:related-statistics} and \ref{sec:related-ML}, respectively, in view of own prior work briefly reported in Section \ref{sec:own-work} which combines both viewpoints. The recent related work discussed in Section \ref{sec:related-ML} reflects the fact that generative models for discrete probability distributions has become an active field of research recently.

\subsubsection{Statistics}\label{sec:related-statistics}
Joint distributions of discrete random variables have a long history in multivariate statistics \cite{Agresti:2013aa}. This includes the study of subsets of such distributions known as \textit{discrete graphical models} \cite{Lauritzen:1996aa,Cowell:1999aa,Koller:2009aa}. Here, conditional independency assumptions encoded by the structure of an underlying graph \cite{Studeny:2005aa} effectively reduce the degree of freedoms \eqref{eq:def-N-cn} of general discrete distributions $p$ and imply their factorization of once realizations of conditioning variables are observed. From the algebraic viewpoint, such statistical assumptions about $p$ give rise to monomial constraints. The study of the topology and geometry of the resulting algebraic varieties which support corresponding subfamilies of distributions, is the subject of the fields of \textit{algebraic statistics} \cite{Geiger:2006aa,Lin:2009aa,Drton:2009aa,Zwiernik:2016aa,Sullivant:2018aa}. The special case of fully factorizing discrete distributions
\begin{equation}\label{eq:p-alpha-factorized}
p(\alpha) = \prod_{i\in[n]} p_{i}(\alpha_{i})
\end{equation}
is particularly relevant for this paper. For example, the subfamily of all such distributions for the toy case $n=c=2$,  depicted by Figure \ref{fig:Wright-AF}, is known as Wright manifold in mathematical game theory \cite{Hofbauer:1998aa} and more generally as \textit{Segre variety} $\Sigma_{1,1}$ in algebraic geometry \cite{Harris:1992ub,Landsberg:2012aa}. 

\subsubsection{Own Prior Work}\label{sec:own-work}
Our approach utilizes \textit{assignment flows} \cite{Astrom:2017ac} that evolve on the relative interior of the product of $n$ probability simplices $\Delta_{c}$, called \textit{assignment manifold}, one factor for each random variable $y_{i},\;i\in[n]$ conforming to the factorization \eqref{eq:p-alpha-factorized}. As summarized in Section \ref{sec:AFs}, the restriction to strictly positive discrete distribution with full support  enables to turn these domains into elementary statistical manifolds equipped with the Fisher-Rao geometry and the e-connection of information geometry \cite{Amari:2000aa}. The corresponding exponential map and the geodesics can be specified in closed form. 

Assignment flows are turned into a generative model for discrete random variables as illustrated by Figure \ref{fig:approach}, which generalizes the toy example \eqref{eq:py1y2} and Figure \ref{fig:Wright-AF}: 
Geometric integration of the assignment flow realizes a map which pushes forward a standard reference measure on the tangent space at the barycenter to the extreme points of the (closure) of the assignment manifold. By embedding the assignment manifold into the simplex \eqref{def:Delta-N} of all discrete joint distributions, the pushforward measure concentrates on the extreme points and hence represents a more complex \textit{non-factorizing} discrete joint distribution by convex combination of Dirac measures. 

Our recent work \cite{Boll:2024aa} characterizes assignment flows as multi-population games and studies multi-game dynamics via the aforementioned embedding approach. Some results established in this work regarding the embedding map will be employed in Section \ref{sec:structured_flow_matching}.

\subsubsection{Machine Learning}\label{sec:related-ML}

The lack of work on generative models for \textit{discrete} distributions stated in the survey papers \cite{Kobyzev:2019aa,Papamakarios:2021vu} has stimulated corresponding research recently. 

The paper \cite{Stark:2024aa} employs the parametric Dirichlet distribution on the probability simplex \cite{Ferguson:1973aa,Johnson:1977aa,Aitchinson:1982vt} as intermediate conditional distributions in a flow matching approach. 
A similarity to our method is the use of infinite transport time, which achieves favorable scaling in the regime of many classes.
A detailed comparison is discussed in Section~\ref{sec:comparison_to_stark}.

The paper \cite{Davis:2024aa} refers to \cite{Astrom:2017ac} and a preliminary version \cite{Boll:2024ab} of our generative model and uses geodesics with respect to the Riemannian connection rather than the e-connection, corresponding to $\alpha=0$ and $\alpha=1$ in the family of $\alpha$-connections, respectively \cite{Amari:2000aa}. By virtue of the sphere map \cite[Def.~1]{Astrom:2017ac} as isometry, the former geodesics on the simplex correspond to the geodesics (great circles) on the sphere with radius 2, restricted to the intersection with the open positive orthant. The authors of [DKP+24] argue that their approach avoids numerical instability at the boundary of the manifold, which is indeed relevant when working on the sphere. However, this issue does not arise on the simplex either, provided that proper geometric numerical integration schemes are used, as demonstrated in [ZSPS20]. The focus of \cite{Davis:2024aa} is on improving the training dynamics using optimal transport, due to the close relation on the simplex of the geometry induced by the Wasserstein distance and the Fisher-Rao geometry \cite{Li:2018ad}. 

Another line of research, called \textit{dequantization},  concerns the approximation of \textit{discrete} probability distributions by \textit{continuous} distributions \cite{Uria:2013,Theis:2015,Dinh:2017,Salimans:2017,Ho:2019}. A dequantization approach for general discrete data, i.e.~similar in scope to our approach, was recently proposed by \cite{Chen:2022aa}. We discuss this paper in Section \ref{sec:dequantization} and point out differences by showing that our approach can be characterized as dequantization procedure. In particular, we indicate that a key component of the approach \cite{Chen:2022aa}, learning an embedding of class configurations, can be replicated using our approach, by defining an payoff function of our generative assignment flow approach accordingly.

Regarding the training of our generative model, our approach builds on the recent work \cite{Lipman:2023aa,Chen:2023}.
The authors introduced a \textit{flow-matching approach} to the training of continuous generative models which enables more stable and efficient training and hence an attractive alternative to established maximum likelihood training. We adopt this criterion and adapt it to our generative model for discrete distributions and the underlying geometry. In particular, we encode given training data as e-geodesics on the assignment manifold which makes flow matching convenient and effective. 


\subsection{Organization}  
Section \ref{sec:basic-notation} fixes the basic notation. Section \ref{sec:Background} summarizes the assignment flow approach and specifies the flow embedding into the simplex \eqref{def:Delta-N}, along with mappings and their properties required in the remainder of this paper.

The core Section \ref{sec:Approach} introduces and details our approach. Section \ref{sec:Generative-Model} introduces the generative model. 
The flow-matching approach is described in Section \ref{sec:learning} and how it relates to the recent work \cite{Lipman:2023aa,Chen:2023} which inspired the training component of our approach. Section \ref{sec:numerics} details the particular geometric integration used in all experiments for computing the assignment flow, based on the methods worked out by \cite{Zeilmann:2020aa}. Section \ref{sec:Likelihood} explains how the trained generative model is evaluated for computing the likelihoods of a novel unseen data points. Section \eqref{sec:dequantization} explains dequantization and characterizes our approach from this viewpoint.

Experimental results are presented and discussed in Section \ref{sec:Experiments}. We conclude in Section \ref{sec:Conclusion}.

\subsection{Basic Notation, List of Main Symbols}\label{sec:basic-notation}
We set $[n]:=\{1,2,\dotsc,n\}$ for $n\in\N$. The canonical Euclidean inner product as well as the matrix inner product which induces the Frobenius norm, are denoted by $\la\cdot,\cdot\ra$. The mapping $\Diag(\cdot)$ takes a vector to the diagonal matrix with the vector component as main diagonal entries. $e_{k},\,k\in\N$, denotes a unit vector with single non-zero $k$-th component equal to $1$ and dimension, that is clear from the context.

\textbf{Data, labelings.} $\mc{G}=(\mc{V},\mc{E})$, with vertex set $\mc{V}=[n]$,  denotes an arbitrary graph on which data $x_{i}$ are observed at every vertex $ i\in\mc{V}$. $c\in\N$ possible class labels of the data $x_{i}$ are represented by discrete random variables $y_{i}\in[c]$. Realizations of the variables $y_{i}$ are denoted by $\alpha_{i}\in [c]$. This results in $N=c^{n}$ labelings configurations $\alpha=\{\alpha_{1},\dotsc,\alpha_{n}\}$ for given data $x=\{x_{1},\dotsc,x_{n}\}$.

\textbf{Assignment flows, dynamical labelings.} The probability simplex is denoted by 
\begin{equation}\label{eq:def-Delta-n}
\Delta_{n}=\{p\in\R_{\geq 0}^{n}\colon\la\eins_{n},p\ra=1\}, 
\end{equation}
where $\eins_{n}=(1,1,\dotsc,1)^{\T}\in\R^{n}$. \textit{Assignment flows} (Section \ref{sec:AFs}) work with the relative interior $\mathring{\Delta}_{c}$ of $\Delta_{c}$, denoted by $\mc{S}_{c}:=\mathring{\Delta}_{c}$, containing the strictly positive probability vectors of dimension $c$, and with the $n$-fold product conforming to $\mc{V}$,
\begin{equation}\label{eq:def-mcW}
\mc{W}_{c}
:= \mc{S}_{c}\times\dotsb\times\mc{S}_{c},
\qquad\qquad
(\text{$n=|\mc{V}|$ factors})
\qquad\qquad (\text{assignment manifold})
\end{equation}
Points on $\mc{W}$ are denoted by 
\begin{equation}
W
= (W_{1},\dotsc,W_{n})^{\T}\in \mc{W}_{c} \subset \R_{>0}^{n\times c},\qquad W_{i}\in\mc{S}_{c},\qquad i\in[n].
\end{equation}
The evolution $W(t)$ of these assignment vectors, obtained by integrating the assignment flow equation, determines the label assignments $\alpha_{i}$ to the data point $x_{i}$ at every $i\in\mc{V}$, by convergence to the corresponding unit vectors
\begin{equation}\label{eq:lim-t-Wt}
\lim_{t\to\infty} W_{i}(t)=e_{\alpha_{i}}\in\{0,1\}^{c},\quad i\in\mc{V},
\end{equation}
which are the extreme points of the closure of the assignment manifold $\ol{\mc{W}_{c}}$. Further spaces and mappings defined in connection with assignment flows in Section \ref{sec:AFs} are: The tangent spaces $T_{0}, \mc{T}_{0}$ to $\mc{S}_{c}, \mc{W}_{c}$ with orthogonal projections $\pi_{0}, \Pi_{0}$, the barycenters $\eins_{\mc{S}}, \eins_{\mc{W}}$ of $\mc{S}_{c}, \mc{W}_{c}$, the Fisher-Rao metric $g_{p}, g_{W}$ on $T_{0}, \mc{T}_{0}$, the replicator maps $R_{p}, R_{W}$ and the lifting maps $\exp_{p}, \exp_{W}$ which play the role of exponential maps.

Besides the underlying geometry, the essential part of the assignment flow equation, whose integration results in \eqref{eq:lim-t-Wt}, is the
\begin{equation}\label{eq:def-affinity-function}
F_{\theta}\colon\mc{W}_{c}\to\R^{n\times c},
\qquad\qquad(\text{affinity function})
\end{equation}
whose parameters $\theta$ are learned from data.

\textbf{Meta-simplex, assignment manifold embedding.} We overload the symbol $p$ to denote discrete probability distributions using any of the equivalent representations specified after Eq.~\eqref{def:Delta-N}, as well as discrete probability vectors whose dimension should be unambigous from the context. Major examples are $p\in\mc{S}_{c}\subset\R_{\geq 0}^{c}$ and $p\in\Delta_{N}$ (cf.~\eqref{def:Delta-N}).

Since the embedding
\begin{equation}\label{eq:TWc}
\mc{T} := T(\mc{W}_{c}) \subset\mc{S}_{N}:=\mathring{\Delta}_{N}
\qquad\qquad(\text{meta-simplex embedding})
\end{equation}
of the assignment manifold defined in Section \eqref{sec:Meta-Simplex} yields the submanifold of factorizing distributions in $\Delta_{N}$, as depicted for a toy scenario by Figure \ref{fig:Wright-AF}(a), we call $\Delta_{N}$ as defined by \eqref{def:Delta-N} ``meta-simplex'', to distinguish the product of simplices $\mc{W}_{c}$ \eqref{eq:def-mcW} before and after the embedding $T(\mc{W}_{c})$ \eqref{eq:TWc}.

We denote by
\begin{equation}
\mc{P}(\mc{S}_{c}), \; \mc{P}(\mc{W}_{c}), \; \text{etc.}
\end{equation}
the set of probability measures supported on the space $\mc{S}_{c}, \mc{W}_{c}$, etc.

\section{Background}\label{sec:Background}

Section \ref{sec:AFs} defines spaces and mappings required in the remainder of the paper. Section \ref{sec:Meta-Simplex} defines a key ingredient of our approach, the embedding \eqref{eq:TWc} and related mappings. We refer to the basic notation introduced in Section \ref{sec:basic-notation}.

\subsection{Assignment Flows}\label{sec:AFs}
The basic state space of discrete distributions is the relative interior of the probability simplex
\begin{subequations}
\begin{align}
\label{eq:def-mcSc}
\mc{S}_{c}
&:=\mathring{\Delta}_{c} = \{p\in\R^{c}\colon p_{j}>0,\;\la\eins_{c},p\ra=1,\;\forall j\in[c]\}
\intertext{with its}
\eins_{\mc{S}} 
&:= \frac{1}{c}\eins_{c}\in\mc{S}_{c},
&&
(\text{barycenter})
\intertext{
which becomes the Riemannian manifold $(\mc{S}_{c},g)$ with trivial tangent bundle $T\mc{S}_{c} = \mc{S}_{c}\times T_{0}$, comprising the}
\label{eq:def-T0}
T_{0} &:= T_{\eins_{\mc{S}}}\mc{S}_{c} 
:= \{v\in\R^{c}\colon \la\eins_{c},v\ra=0\}
&&
(\text{tangent space})
\intertext{with the orthogonal projection}\label{eq:def-pi0}
\pi_{0}\colon \R^{c} &\to T_{0},\qquad
\pi_{0} := I_{c}-\eins_{c}\eins_{\mc{S}}^{\T}
&&(\text{orthogonal projection})
\intertext{and carrying the}
\label{eq:FR-simplex}
g_{p}(u,v) 
&:= \la u,\Diag(p)^{-1} v\ra,\qquad u,v\in T_{0},\quad p\in\mc{S}_{c}.
&&
(\text{Fisher-Rao metric})
\end{align}
\end{subequations}
This naturally extends to the product manifold $(\mc{W}_{c},g)$ given by \eqref{eq:def-mcW}, with trivial tangent bundle $T\mc{W}_{c}=\mc{W}_{c}\times\mc{T}_{0}$, and
\begin{subequations}
\begin{align}
\eins_{\mc{W}} 
&= (\eins_{\mc{S}},\dotsc,\eins_{\mc{S}})^{\T},
&&(\text{barycenter})
\\
\label{eq:def-mcT0}
\mc{T}_{0}
&:= T_{\eins_{\mc{W}}}\mc{W}_{c} := T_{0}\times\dotsb\times T_{0},
\qquad(\text{$n=|\mc{V}|$ factors})
&&(\text{tangent space})
\intertext{with points denoted by}
V
&= (V_{1},\dotsc,V_{n})^{\T}\in\R^{n\times c}
\in\mc{T}_{0},\qquad
V_{i}\in T_{0},\qquad i\in[n],
\intertext{the orthogonal projection}
\label{eq:def-Pi0}
\Pi_{0}\colon\R^{n\times c} &\to\mc{T}_{0},\qquad
\Pi_{0} U := (\pi_{0} U_{1},\dotsc, \pi_{0} U_{n})^{\T}
&&(\text{orthogonal projection})
\intertext{and the}\label{eq:def-gW}
g_{W}(U,V) &= \sum_{i\in[n]} g_{W_{i}}(U_{i},V_{i}),\qquad 
U, V\in\mc{T}_{0},\quad W\in\mc{W}_{c}.
&&(\text{Fisher-Rao metric})
\end{align}
\end{subequations}
\textit{Assignment flows} are dynamical systems of the general form
\begin{equation}\label{eq:AF-general}
\dot W(t) = R_{W(t)}\big[F_{\theta}\big(W(t)\big)\big],\qquad
W(0) = W_{0} \in\mc{W}_{c},
\qquad\qquad(\text{assignment flow})
\end{equation}
parametrized by an affinity function \eqref{eq:def-affinity-function} and comprising the linear mappings
\begin{subequations}\label{eq:def-replicator-maps}
\begin{align}
R_{p}\colon\R^{c}&\to T_{0},\qquad
R_{p} = \Diag(p)-p p^{\T},\qquad p\in\mc{S}_{c}
&&(\text{replicator map})
\\
R_{W}\colon\R^{n\times c}&\to\mc{T}_{0},\qquad
R_{W}[F_{\theta}] = (R_{W_{1}}F_{\theta,1},\dotsc,R_{W_{n}}F_{\theta,n})^{\T},\qquad W\in\mc{W}_{c}.
&&(\text{replicator map})
\end{align}
\end{subequations}
The \textit{exponential maps} with respect to the e-connection reads
\begin{subequations}\label{eq:def-Exp}
\begin{align}
    \Exp_{p}(v) &= \frac{p\cdot e^{\frac{v}{p}}}{\la p, e^{\frac{p}{v}}\ra},
& p&\in\mc{S}_{c},\quad v\in T_{0},
\\
\Exp_{W}(V) &= \big(\Exp_{W_{1}}(V_{1}),\dotsc,\Exp_{W_{n}}(V_{n})\big)^{\T}
& W&\in\mc{W}_{c},\quad V\in\mc{T}_{0},
\end{align}
\end{subequations}
where both the multiplication $\cdot$ and the exponential function apply componentwise. Composition with the replicator maps \eqref{eq:def-replicator-maps} yields the
\begin{subequations}\label{eq:def-lifting-maps}
\begin{align}
\exp_{p}\colon T_{0}&\to\mc{S}_{c},\qquad\;
\exp_{p} := \Exp_{P}\circ R_{p},\qquad\qquad p\in\mc{S}_{c},
&&(\text{lifting map})
\label{eq:def-lifting-p} \\ \label{eq:def-lifting-W}
\exp_{W}\colon\mc{T}_{0}&\to\mc{W}_{c},\qquad
\exp_{W} := \Exp_{W}\circ R_{W},\qquad W\in\mc{W}_{c}.
&&(\text{lifting map})
\end{align}
\end{subequations}

\subsection{Meta-Simplex, Flow Embedding}\label{sec:Meta-Simplex}
The embedding \eqref{eq:TWc} is defined by the map 
\begin{equation}\label{eq:def-T-embedding}
T\colon\mc{W}_{c}\to\mc{T}=T(\mc{W}_{c}) \subset\mc{S}_{N},\qquad
T(W)_{\alpha} := \prod_{i\in[n]} W_{i,\alpha_{i}},\qquad
\alpha\in[c]^{n}.
\end{equation}
Denoting the tangent space to $\mc{S}_{N}$ defined by \eqref{eq:TWc} by
\begin{equation}
\mc{T}_{0}\mc{S}_{N} := \{z\in\R^{N}\colon \la\eins_{N},z\ra=0\},
\qquad\qquad
(\text{meta-tangent space})
\end{equation}
we also require the map
\begin{equation}\label{eq:def-Q-embedding}
Q\colon\R^{n\times c}\to\R^{N},\qquad
Q\colon \mc{T}_{0}\to \mc{T}_0\mc{S}_{N},\qquad
(QV)_{\alpha} := \sum_{i\in[n]} V_{i,\alpha_{i}},\qquad
\alpha\in[c]^{n}.
\end{equation}
The mappings $T, Q$ have been studied by \cite{Boll:2023aa,Boll:2024aa}.

Every point $W\in\mc{W}$ on the assignment manifold is represented through \eqref{eq:def-T-embedding} by the combinatorially large vector $T(W)$ with $N=c^{n}$ components $T(W)_{\alpha}$, consisting of monomials of degree $n$ in the variables $W_{i,\alpha_{i}}\in (0,1)$. A labeling  determined by the assignment flow by \eqref{eq:lim-t-Wt} corresponds to 
\begin{equation}
\lim_{t\to\infty} T\big(W(t)\big) 
= T\big((e_{\alpha_{1}},\dotsc,e_{\alpha_{n}})^{\T}\big)=e_{\alpha},
\end{equation}
that is, the unit vector (vertex) of the meta-simplex $\Delta_{N}=\ol{\mc{S}_{N}}$ corresponding to the Dirac measure $\delta_{\alpha}$ concentrated on the labelling $\alpha \in [c]^{n}$.
\begin{example}\label{ex:WrightManifold}
We reconsider the toy scenario \eqref{eq:py1y2} of joint distributions of two binary variables. Such distributions correspond on the assignment manifold to points of the form
\begin{equation}
W = \Big(\bsm w_{1} \\ 1-w_{1} \esm, \bsm w_{2} \\ 1-w_{2} \esm\Big)^{\T},\qquad
w_{1}, w_{2}\in (0,1).
\end{equation}
Embedding this point by  \eqref{eq:def-T-embedding} yields the vector 
\begin{equation}
T(W)=\big(w_{1} w_{2}, w_{1}(1-w_{2}), (1-w_{1}) w_{2}, (1-w_{1})(1-w_{2})\big)^{\T}, 
\end{equation}
with components $T(W)_{\alpha}$ indexed by the four possible labeling $\alpha\in\{(1,1),(1,0),(0,1),(0,0)\}$. Since any distribution on the assignment manifold factorizes, this vector is determined by merely two parameters $w_{1}, w_{2}$. Accordingly, the embedded assignment manifold $\mc{T}=T(\mc{W}_{c})\subset\mc{S}_{N}$ is the two-dimensional submanifold depicted by Figure \ref{fig:Wright-AF}(a).

In mathematics, such embedded sets are known as \textit{Segre varieties} at the intersection of algebraic geometry and statistics \cite{Lin:2009aa, Drton:2009aa}.
\end{example}
%
%
The following proposition highlights the specific role of the submanifold of $\mc{S}_{N}$ corresponding to the \textit{embedded assignment manifold} $\mc{T}=T(\mc{W})\subset\mc{S}_{N}$.
\begin{proposition}[{\cite[Prop.~3.2]{Boll:2024aa}}]
For every $W\in\mc{W}_{c}$, the distribution $T(W)\in\mc{S}_{N}$ has maximum entropy 
\begin{equation}
H\big(T(W)\big) = -\sum_{\alpha\in[c]^{n}} T(W)_{\alpha}\log T(W)_{\alpha}
\end{equation}
among all $p\in\mc{S}_{N}$ subject to the marginal constraint
\begin{subequations}
\begin{align}\label{eq:marginal-p}
M p &= W,
\intertext{where the marginalization map is given by}\label{eq:def-M-map}
M\colon\R^{N}\to\R^{n\times c},
\qquad
(M p)_{i,j} &:= \sum_{\alpha\in[c]^{n}\colon \alpha_{i}=j} p_{\alpha},\qquad
\forall (i,j)\in [n]\times [c].
\end{align}
\end{subequations}
\end{proposition}
As a consequence, any \textit{general} distribution 
$p\in\mc{S}_{N}\setminus T(\mc{W}_{c})$ 
which is \textit{not} in $T(\mc{W}_{c})$, has \textit{non}-maximal entropy and hence is \textit{more} informative by encoding additional statistical dependencies \cite{Cover:2006aa}. 

Our approach for \textit{generating} such general distributions $p\in\mc{S}_{N}$, by combining simple factorizing distributions $W\in\mc{W}_{c}$ via the embedding \eqref{eq:def-T-embedding} and assignment flows \eqref{eq:AF-general}, is introduced in following Section \ref{sec:Approach}.

\section{Approach}\label{sec:Approach}

Section \ref{sec:Generative-Model} introduces our generative model for representing and learning a discrete joint distribution $p=p(\alpha)\in \mc{S}_N$ of label configurations $\alpha=(\alpha_{1},\dotsc,\alpha_{n})$ as realizations of discrete random variables $y=(y_{1},\dotsc,y_{n})\sim p$. The approach is illustrated by Figure~\ref{fig:approach}. The training procedure for simulation-free training of the generative model is worked out in Section \ref{sec:learning}. Section \ref{sec:structured_flow_matching} specifies precisely how the approximation of $p$ is achieved in the meta-simplex by measure transport on the embedded nonlinear submanifold of factorizing distributions.

We conclude with short Sections \ref{sec:numerics}--\ref{sec:dequantization} on the geometric integration method that we employed for the discretization of our time-continuous generative model in numerical experiments, on the computation of the likelihood $\wt{p}(\alpha)$ of arbitrary label configurations using the learned generative model, and on the characterization of our approach as a dequantization procedure.

\begin{figure}
    \centering
    \includegraphics[width=0.5\textwidth]{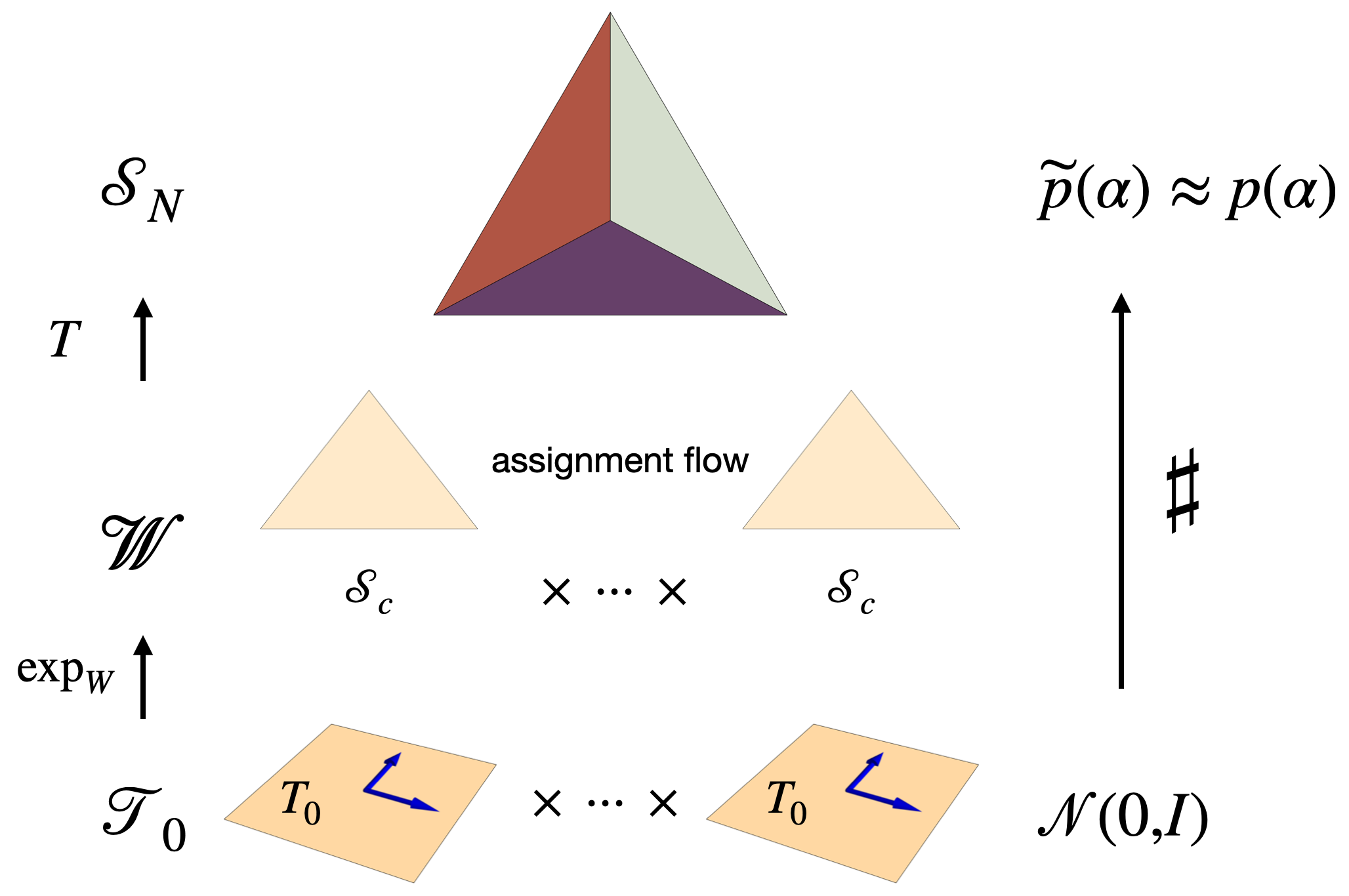}
    \caption{\textbf{Overview of the approach:} The standard Gaussian reference measure $\mc{N}(0,I)$ is pushed forward by the lifting map $\exp_{W}$ from the flat tangent product space $\mc{T}_{0}$ to the assignment manifold $\mc{W}_{c}$, and further to the meta-simplex $\mc{S}_{N}$ via the embedding map $T$ \eqref{eq:def-T-embedding}, by geometrically integrating the assignment flow equation \eqref{eq:AF-general}. Since the assignment flow converges to the extreme points of $\ol{\mc{W}_{c}}$ which after embedding agree with the extreme points of $\Delta_{N}=\ol{\mc{S}_{N}}$, an approximation $\wt{p}(\alpha)$ of a general \textit{discrete} target measure $p(\alpha)$ can be learned in terms of a corresponding convex combination of extreme points. This is achieved by matching the flow of e-geodesics which encode given training samples to the generating assignment flow, by empirical expectation, and by learning the parameters of the affinity function $F_{\theta}$ \eqref{eq:def-affinity-function}. Since factorizing distributions $T(W),\, W\in\mc{W}_{c}$, are only required, the approach is computationally feasible also in high dimensions.}
    \label{fig:approach}
\end{figure}

\subsection{Generative Model}\label{sec:Generative-Model}
\subsubsection{Goal}
The goal is to learn an approximation 
\begin{equation}
\wt{p}\approx p, 
\qquad\qquad
(\text{approximation})
\end{equation}
as convex combination of \textit{factorizing} joint distributions. The submanifold $\mc{T}=T(\mc{W}_{c})\subset \mc{S}_N$ shown in Figure \ref{fig:Wright-AF}(a) spans all factorizing distributions $T(W)\in \mc{S}_N$, which are efficiently represented by their marginals $W\in \mc{W}_{c}$ due to \eqref{eq:marginal-p}. In particular, since the dimension of $\mc{W}_{c}$ only grows linearly in the number of variables $n$, factorizing distributions are tractable to work with numerically. 
However, only \textit{independent} random variables follow factorizing distributions, posing the question of how statistical \textit{coupling} between such variables can be represented through convex combination. 

\subsubsection{Representation of General Distributions}
Note that the submanifold of factorizing distributions $\mc{T}\subseteq\mc{S}_N$ is \textit{nonconvex}. Thus, convex combinations of two factorizing distributions $T(W_1)$ and $T(W_2)$ generally lie \textit{outside} of $\mc{T}$ and hence form a \textit{non-factorizing} distribution.

In addition, we observe that every Dirac measure $e_\alpha$ factorizes. Intuitively, this is because each variable has a deterministic value, independent of all others.
Because Dirac measures are the extreme points of the convex set $\ol{\mc{S}_N}$, \emph{every} joint distribution $\wt{p}\in \mc{S}_N$ representing an \emph{arbitrary} coupling between variables can be written as a convex combination of Dirac measures
\begin{equation}\label{eq:dirac_convex_comb}
    \wt{p} = \sum_{\alpha\in [c]^n} \wt{p}_\alpha e_\alpha.
\end{equation}
This particular representation of $\wt{p}$ is intractable, however, because it involves a combinatorially large number of mixture coefficients $\wt{p}_\alpha$.
To tame this complexity, the \textbf{key idea} is to \textit{represent mixtures $\wt{p}\in \mc{S}_N$ of factorizing distributions as measures $\nu\in \mc{P}(\mc{W}_{c})$} by
\begin{equation}\label{eq:mixture_representation_nu}
    \wt{p} = \EE_{W\sim\nu}[T(W)].
\end{equation}
This shifts the problem of parameterizing useful subsets of combinatorially many mixture coefficients in \eqref{eq:dirac_convex_comb} to the problem of parameterizing a preferably large subset of measures $\nu\in \mc{P}(\mc{W}_{c})$, supported on the comparatively low-dimensional manifold $\mc{W}_{c}$. 
The latter can be achieved by \textit{parameterized measure transport} on the \textit{assignment manifold} $\mc{W}_{c}$. 

Specifically, a simple reference measure 
\begin{subequations}
\begin{align}
\nu_0&\in \mc{P}(\mc{W}_{c})
&&(\text{reference measure})
\intertext{
is chosen and transported by the assignment flow \eqref{eq:AF-general}, reaching 
}
\nu &= \nu_\infty\quad\text{for}\quad
t\to\infty.
&&(\text{transported measure})
\intertext{
\textit{Parameterization} of measures 
}
\nu_\theta&\in \mc{P}(\mc{W})
&&(\text{parametrized measure})
\end{align}
\end{subequations}
is achieved by choosing an appropriate class of affinity functions $F_\theta\colon \mc{W}\to\R^{n\times c}$ \eqref{eq:def-affinity-function} driving the assignment flow \eqref{eq:AF-general}. 
Note that, while the support of $\wt{p}$ in \eqref{eq:dirac_convex_comb} was directly associated with the number of mixture coefficients, the complexity of representing $\wt{p}$ via the ansatz \eqref{eq:mixture_representation_nu} is no longer associated with its support. 

The simplest example of \eqref{eq:mixture_representation_nu} is the representation of 
\begin{equation}\label{eq:wtp-eins-SN}
\wt{p}=\eins_{\mc{S}_{N}}
\end{equation}
by choosing $F_\theta \equiv 0$ and a \textit{product} reference distribution 
\begin{equation}\label{eq:nu0-product}
\nu_0 
= \prod_{i\in [n]} \nu_{0;i} \in\mc{P}(\mc{W}_{c})
\qquad\qquad
\end{equation}
with mean $\EE_{W_i\sim \nu_{0;i}}[W_i] = \eins_{\mc{S}_c}$, which through the embedding \eqref{eq:mixture_representation_nu} yields \eqref{eq:wtp-eins-SN}, which has \textit{full} support on the very high-dimensional space $[c]^{n}$. We make this connection more explicit.
\begin{lemma}[\textbf{convex combination of embedded nodewise measures}]\label{lem:nodewise_nu}
Suppose the reference measure $\nu_{0}$ has the product form \eqref{eq:nu0-product} 
with $\nu_{i}\in \mc{P}(\mc{S}_c)$. Then the joint distribution represented by the mixture \eqref{eq:mixture_representation_nu} reads
\begin{equation}\label{eq:mixture_independent_measures}
    \wt{p} = \EE_{W\sim\nu}[T(W)] = T(\wh{W}),\qquad \wh{W}_i = \EE_{W_i\sim \nu_i}[W_i],\quad i\in [n].
\end{equation}
\end{lemma}
\begin{proof}
Let $\alpha \in [c]^n$ be an arbitrary multi-index. Since $\nu$ factorizes in the described manner, $W\sim \nu$ is independently distributed on each node which implies
\begin{equation}\label{eq:nodewise_nu_proof}
\wt{p}_{\alpha} = \EE_{W\sim\nu}[T(W)_\alpha] = \EE_{W\sim\nu}\Big[\prod_{i\in [n]} W_{i,\alpha_i}\Big]
    = \prod_{i\in [n]} \EE_{W\sim\nu_i}[W_{i,\alpha_i}] 
    = T(\wh{W})_\alpha.
\end{equation}
\end{proof}
Lemma~\ref{lem:nodewise_nu} shows that, if $\nu$ is independent on every node, then $\wt{p}\in \mc{T}$. In particular, 
\textit{coupling} between variables, to be represented by the joint distribution $\wt{p}$, has necessarily to be induced by the \textit{interaction} of node states over the course of integrating the assignment flow.

\subsubsection{Model Learning and Model Evaluation (Sampling)}
The target distribution $p$ is unknown, in practice, and only independently drawn training samples $\beta\sim p$ are available. 
After choosing a class of payoff functions $F_{\theta}$, the task is to learn parameters $\theta$ such that 
\begin{equation}
\wt{p} = \EE_{W\sim\nu_\theta}[T(W)],
\end{equation}
i.e.~a parametrization of the right-hand side of \eqref{eq:mixture_representation_nu}, 
approximates the empirical distribution of samples. To this end, we identify samples $\beta$ with the corresponding extremal points $M e_{\beta}\in\ol{\mc{W}_{c}}$ (Section~\ref{sec:RM-equation}) and use \emph{flow matching} on $\mc{W}_{c}$ to learn $\theta$ in a numerically stable and efficient way (Section~\ref{sec:learning}).

After learning has converged, new samples $\beta\sim\wt{p}$ from the approximate distribution $\wt{p}\approx p$ can be drawn by a two-stage process: 
\begin{enumerate}[(i)]
\item
First, an initialization $W_0\sim \nu_0$ is drawn and evolved over time $W(t) \in \mc{W}_{c}$ by integrating the learned assignment flow until either the desired time $t_{\max}$ is reached, or $W(t)$ approaches an extreme point of $\ol{\mc{W}_{c}}$. 
\item
The new data is subsequently drawn from the factorizing distribution $T(W(t_{\max}))$. At extreme points $M e_{\beta'}$, this distribution is a Dirac measure and sampling from it always yields $\beta'$.
\end{enumerate}

\subsection{Riemannian Flow Matching}
\label{sec:learning}

in this section, we work out details of the procedure for training generative assignment flows.
\subsubsection{Representation of Labelings as Training Data}
\label{sec:RM-equation}
Our approach to training the generative model utilizes \textit{labelings} as training data of the form
\begin{equation}
\ol{W}\in\ol{\mc{W}_{c}},\qquad
\ol{W}_{i} = e_{\alpha_{i}},\qquad \alpha_{i}\in[c],\qquad\forall i\in[n].
\end{equation}
Any such point $\ol{W}$ assigns a label (category) $\alpha_{i}$ to each vertex $i\in\mc{V}$ in terms of a corresponding unit vector $e_{\alpha_{i}}\in\{0,1\}^{c}$. 
The flow-matching criterion, specified in the following section, is optimized to find $\theta$ such that $F_\theta$ drives the assignment flow to labelings in the limit $\lim_{t\to\infty} W(t) = \ol{W}$. 
In practice, the assignment flow is integrated up to a sufficiently large point of time
\begin{equation}\label{eq:def-Tmax}
t_{\max} > 0
\end{equation}
followed by trivial rounding of $W_{i}(t_{\max})\mapsto e_{\alpha_{i}}$ at every vertex $i$. 

\subsubsection{Training Criterion}\label{sec:training-criterion}
This section details the approach schematically depicted by Figure \ref{fig:approach}. In the following, 
\begin{equation}
\beta\sim p
\end{equation}
denotes labeling configurations for training, drawn from the \textit{unknown} underlying discrete joint data distribution $p$. $\beta$ corresponds to the Dirac measure $e_{\beta}\in\mc{S}_{N}$ (extreme point) of the meta-simplex $\mc{S}_{N}$ and to a corresponding point $\ol{W}_{\beta} = M e_{\beta} \in\ol{\mc{W}_{c}}$, to which the assignment flow \eqref{eq:AF-general} may converge.
 
The idea of flow matching is to directly fit the model vector field, in our case the assignment flow vector field \eqref{eq:AF-general},
\begin{equation}\label{eq:def-V-theta}
V_\theta(W, t) := R_{W}[F_\theta(W, t)],
\end{equation}
to a vector field whose flow realizes a desired measure transport. 
Let $\nu_0\in \mc{P}(\mc{W})$ be a simple reference measure and define \textit{conditional probability paths} 
\begin{equation}
\nu_t(\beta)
\end{equation}
satisfying
\begin{subequations}\label{eq:nu_conditional_endpoints}
\begin{align}
 \nu_{0}(\beta) &:=  \nu_{0}
\label{eq:nu0-conditional} \\ 
\label{eq:nu1-conditional}
 \nu_{\infty}(\beta) 
 &:= \lim_{t\to\infty}  \nu_t(\beta) = \delta_{\ol{W}_{\beta}}(W)\quad
 \text{for all}\quad \beta\in [c]^n,
\end{align}
\end{subequations}
where $\ol{W}_{\beta} = (e_{\beta_{1}},\dotsc, e_{\beta_{n}})^{\T}\in\{0,1\}^{n\times c}$ is the extreme point of $\ol{\mc{W}_{c}}$ corresponding to $\beta$, such that $T(\ol{W}_{\beta})=e_{\beta}\in\ol{\mc{S}_{N}}$. Then the \textit{marginal probability path}
\begin{equation}\label{eq:marginal_distr_from_conditionals}
     \nu_t = \EE_{\beta\sim p}[ \nu_t(\beta)]
\end{equation}
represents the target data distribution $p$ in the limit $t\to\infty$ by $\nu_{\infty}$ and
\begin{equation}\label{eq:marginal_limit_representation}
\EE_{W\sim  \nu_\infty}[T(W)] 
= \EE_{\beta\sim p}[e_\beta] = p.
\end{equation}

In principle, we can now define a vector field 
\begin{equation}
u_t\colon \mc{W}_{c}\to \mc{T}_{0}
\end{equation}
which generates the path $t\mapsto\nu_t$ in the sense that the flow of $u_t$ pushes forward $\nu_0$ to $\nu_t$, for all times $t\geq 0$. Let $\rho\in \mc{P}([0,\infty))$ be a distribution with full support on the non-negative time axis.
Regression of the assignment flow vector field \eqref{eq:def-V-theta}, 
\begin{equation}
V_{\theta}(\cdot,t)\colon \mc{W}_{c}\to \mc{T}_{0}, 
\end{equation}
with respect to $u_t$, amounts to minimizing the \textbf{Riemannian flow matching criterion}
\begin{equation}\label{eq:riemannian-flow-matching}
    \mc{L}_\text{RFM}(\theta) = \EE_{t\sim\rho, W\sim \nu_{t}}\Big[\big\|u_{t}(W)-V_\theta(W, t)\big\|_{W}^{2}\Big],
\end{equation}
where $\|\cdot\|_{W}^{2} = g_{W}(\cdot,\cdot)$ (cf.~\eqref{eq:def-gW}). 

In this form, flow matching is intractable, however, because we do not have access to the required field $u_t$.  
On the other hand, since we are at liberty to define \textit{conditional} paths that conform to the constraints \eqref{eq:nu_conditional_endpoints}, we can choose $\nu_t(\beta)$ that are generated by \textit{conditional} vector fields $u_t(\cdot|\beta)$ with \textit{known} form.
The key insight in \cite{Chen:2023}, based on \cite{Lipman:2023aa} and provided that each $\nu_t(\beta)$ is generated by $u_t(\cdot|\beta)$, is that the loss function \eqref{eq:riemannian-flow-matching} has the same gradient with respect to $\theta$ as the \textbf{Riemannian \textit{conditional} flow matching criterion}
\begin{subequations}\label{eq:L-RCFM}
\begin{align}
\mc{L}_{\mrm{RCFM}}(\theta) 
&= \EE_{t\sim\rho, \beta\sim p, W\sim \nu_t(\beta)}\,\Big[\big\|u_t(W|\beta) - V_{\theta}(W,t)\big\|_W^2\Big]
\\
&\overset{\eqref{eq:def-V-theta}}{=} \EE_{t\sim\rho, \beta\sim p, W\sim \nu_t(\beta)}\,\Big[\big\|u_t(W|\beta) - R_{W}[F_\theta(W, t)]\big\|_W^2\Big].
\end{align}
\end{subequations}
By contrast to \eqref{eq:riemannian-flow-matching}, conditional vector fields $u_{t}(W|\beta)$ generating a path \begin{equation}\label{eq:paths-nu-t-beta}
t\mapsto \nu_t(\beta)
\end{equation}
with the required properties \eqref{eq:nu_conditional_endpoints} can be specified in closed form (cf.~Proposition \ref{eq:condflow-model} below), and the conditional loss function \eqref{eq:L-RCFM} can be evaluated efficiently. Ultimately, by minimizing \eqref{eq:L-RCFM}, the measure $\nu_{t}$ generated from the reference measure $\nu_0$ by the assignment flow vector field $R_{W}[F_\theta(W, t)]$ approximates $\nu_\infty$ in the limit $t\to\infty$, which represents the unknown data distribution $p$ through \eqref{eq:marginal_limit_representation}.

\subsubsection{Constructing Conditional Fields}\label{sec:conditional_fields}
This section specifies the conditional vector fields $u_{t}(W|\beta)$ that generate the paths \eqref{eq:paths-nu-t-beta} conforming to \eqref{eq:nu_conditional_endpoints} and define the conditional flow matching objective \eqref{eq:L-RCFM}.

Let 
\begin{equation}\label{eq:def-mcN0}
\mc{N}_{0}(V):=\mc{N}(V;0,\Pi_0)
\end{equation}
denote the standard Gaussian centered in the tangent space at $0\in\mc{T}_{0}$, with the orthogonal projection \eqref{eq:def-Pi0} respresenting the identity map on $\mc{T}_{0}\subset\R^{n\times c}$. 
Pushing forward $\mc{N}_{0}$ by the lifting map \eqref{eq:def-lifting-W} at the barycenter yields a simple \textit{reference distribution} 
\begin{equation}\label{eq:def-nu0}
\nu_0 = (\exp_{\eins_{\mc{W}}})_{\sharp}\mc{N}_{0}\in \mc{P}(\mc{W}).
\end{equation}
The distribution \eqref{eq:def-nu0} is \textit{simple} in the sense that it is easy to draw samples and the conditions of Lemma~\ref{lem:nodewise_nu} are satisfied; in particular, $\nu_0$ \textit{factorizes} node-wise.
For each labeling $\beta\in [c]^n$ and the corresponding extreme point $\ol{W}_{\beta}=(e_{\beta_{1}},\dotsc, e_{\beta_{n}})\in\ol{\mc{W}_{c}}$, and a 
\begin{equation}
\lambda > 0, 
\qquad\qquad(\text{rate parameter})
\end{equation}
define the probability path
\begin{equation}\label{eq:tangent_gaussian_path}
    t\mapsto \mc{N}_{t, \beta} := \mc{N}(\cdot;t \lambda V_\beta,\Pi_0)\in \mc{P}(\mc{T}_{0}),\qquad V_\beta := \Pi_0 \ol{W}_\beta,
\end{equation}
and lift it to $\mc{W}_{c}$, defining
\begin{equation}\label{eq:nu_cond_lifted_gauss}
    \nu_t(\beta) := (\exp_{\eins_{\mc{W}}})_{\sharp} \mc{N}_{t, \beta}.
\end{equation}
The parameter $\lambda$ controls the \textit{rate} at which $\nu_t(\beta)$ moves probability mass closer to $\ol{W}_\beta$. Small values of $\lambda$ move the mass slowly; this is useful in settings with many labels $c \gg 1$, enabling the process to make class decisions during a longer time period. Figure \ref{fig:rate-parameter-lambda} illustrates quantitatively the influence of $\lambda$.
\begin{figure}[htbp]
    \centering
    \includegraphics[width=1.\textwidth]{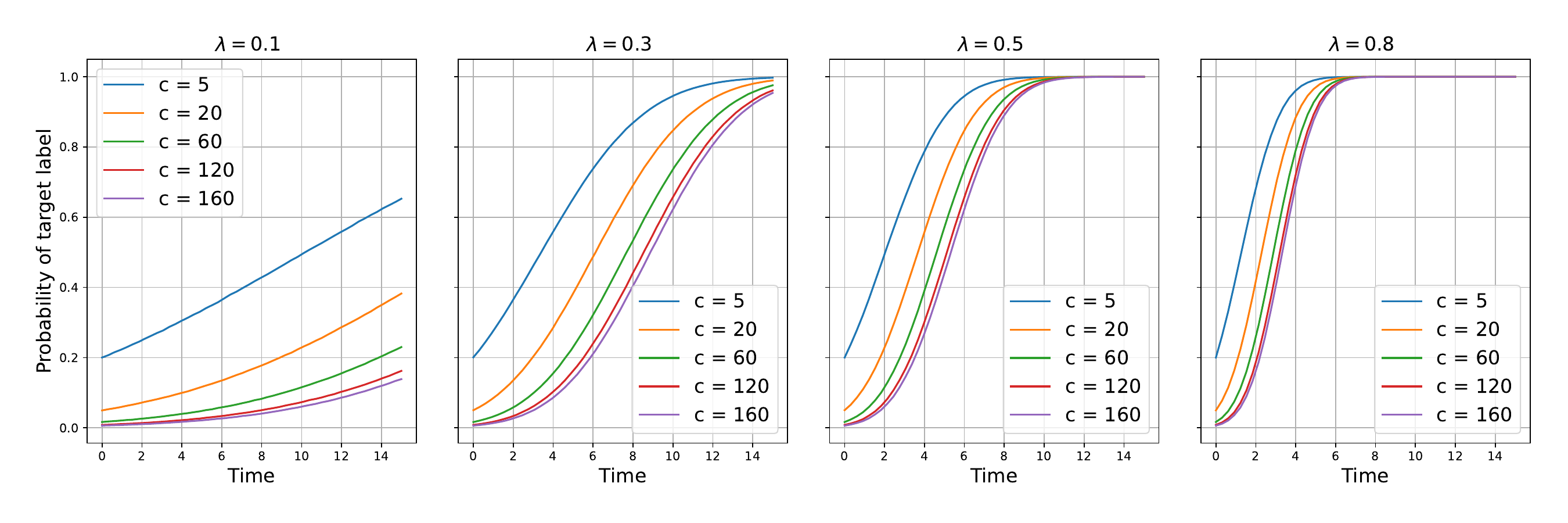}
\caption{Influence of the parameter $\lambda$  controlling in \eqref{eq:tangent_gaussian_path} and \eqref{eq:condvectorfield-model}, respectively, the \textit{rate} of assignment of mass of the pushforward probability measure \eqref{eq:nu_cond_lifted_gauss} to a target label, depending on the number $c$ of labels (classes, categories).}
    \label{fig:rate-parameter-lambda}
\end{figure}


The following proposition makes explicit the conditional vector field $u_{t}(W|\beta)$ that generates \eqref{eq:nu_cond_lifted_gauss} and hence defines the training objective function \eqref{eq:L-RCFM}. Recall the notation of Section \ref{sec:Meta-Simplex} and the first paragraph of Section \ref{sec:training-criterion} explaining the one-to-one correspondence between 
\begin{itemize}
\item a labelling configuration $\beta$, 
\item the corresponding Dirac measure $e_{\beta}\in\mc{S}_{N}$ of the meta simplex, and 
\item the corresponding point $\ol{W}_{\beta}\in\ol{\mc{W}_{c}}$ of the closure of the assignment manifold.
\end{itemize}

\begin{proposition}[\textbf{conditional vector fields}]\label{prop:interpolants}
The probability paths defined in \eqref{eq:nu_cond_lifted_gauss} are generated through the smooth flow
\begin{equation}\label{eq:condflow-model}
    \psi_{\cdot} (\cdot | \beta) \colon \R_{\ge0} \times \mc T_0 \to \mc{W}_{c}, 
    \qquad
    \psi_{t}\big(V | {\beta}\big)  = \exp_{\eins_{\mc W}} (V + t\lambda V_{\beta}).
\end{equation}
It is invertible and has the smooth inverse  
\begin{equation}\label{eq:condflow-inv-model}
    \psi_{t}^{-1}(W| \beta) =\exp_{\eins_{\mc W}}^{-1}(W)-t \lambda V_\beta .
\end{equation}
In particular, the conditional vector field that generates \eqref{eq:nu_cond_lifted_gauss} is given by
\begin{align}\label{eq:condvectorfield-model}
    u_{t}(W |  {\beta}) = R_{W}[\lambda V_\beta].
\end{align}
\end{proposition}
\begin{proof}
See Appendix \ref{sec:proofs-conditional-fields}, page \pageref{sec:proofs-conditional-fields}.
\end{proof}

\begin{proposition}[\textbf{conditional path constraints}]\label{prop:cond_path_constraints}
The conditional probability paths $\nu_t(\beta)$ defined by \eqref{eq:nu_cond_lifted_gauss} satisfy the constraints \eqref{eq:nu_conditional_endpoints}.
\end{proposition}
\begin{proof}
See Appendix \ref{sec:proofs-conditional-fields}, page \pageref{proof-3.3}.
\end{proof}

The path $\mc{N}_{t}$ is generated on the tangent space $\mc{T}_{0}$ by the constant vector field $V\mapsto \lambda V_\beta$ given by \eqref{eq:tangent_gaussian_path}. The related vector field on $\mc{W}_{c}$, which generates the path \eqref{eq:nu_cond_lifted_gauss}, is given by \eqref{eq:condvectorfield-model}. 
Comparing the shape of this field to \eqref{eq:AF-general} makes clear that assignment flows are natural candidate dynamics for matching conditional paths of the described form. The Riemannian conditional flow matching objective \eqref{eq:L-RCFM} consequently reads
\begin{equation}\label{eq:af_flow_matching}
\mc{L}_{\mrm{RCFM}}(\theta) = \EE_{t\sim\rho, \beta\sim p, W\sim \nu_t(\beta)}\,\Big[\big\|R_W[\lambda V_\beta - F_\theta(W, t)]\big\|_W^2\Big].
\end{equation}

We point out that this criterion is `simulation free', i.e.~\textit{no integration} of the assignment flow is required for loss evaluation, which makes training computationally efficient.

Our approach \eqref{eq:af_flow_matching} constitutes a novel instance of the flow-matching approach to generative modeling, introduced by \cite{Lipman:2023aa} and recently extended to Riemannian manifolds by \cite{Chen:2023}. 
This instance uses the assignment manifold \eqref{eq:def-mcW} and the corresponding Riemannian flow \eqref{eq:AF-general}, along with the meta-simplex embedding \eqref{eq:def-T-embedding}, to devise a generative model whose underlying information geometry tailors the model to the representation and learning of \textit{discrete joint} probability distributions.

\subsubsection{Infinite Integration Time}\label{sec:comparison_to_lipman}
A notable difference between our approach and previous Riemannian flow matching methods is that the target distribution is reached for $t\to\infty$ rather than after finite time.
This corresponds to the fact that $e$-geodesics do not reach boundary points of $\ol{\mc{W}_{c}}$ after finite time and thus avoids two problems faced in prior work. 

First, unlike the preliminary version presented in \cite{Boll:2024ab}, data points $\beta\in [c]^n$ do not need to be smoothed in order to present targets in the interior of $\mc{W}$. Instead, we can directly approach extreme points $\ol{W}_\beta\in \ol{\mc{W}_{c}}$, even though they are at infinity in the tangent space $\mc{T}_{0}$ at $\eins_\mc{W}$. Figure \ref{fig:T0-norms} shows that working within a ball in $T_{0}$ with radius $15$ suffices to represent `infinity' in practice. 
\begin{figure}
\centerline{
\includegraphics[width=0.5\textwidth]{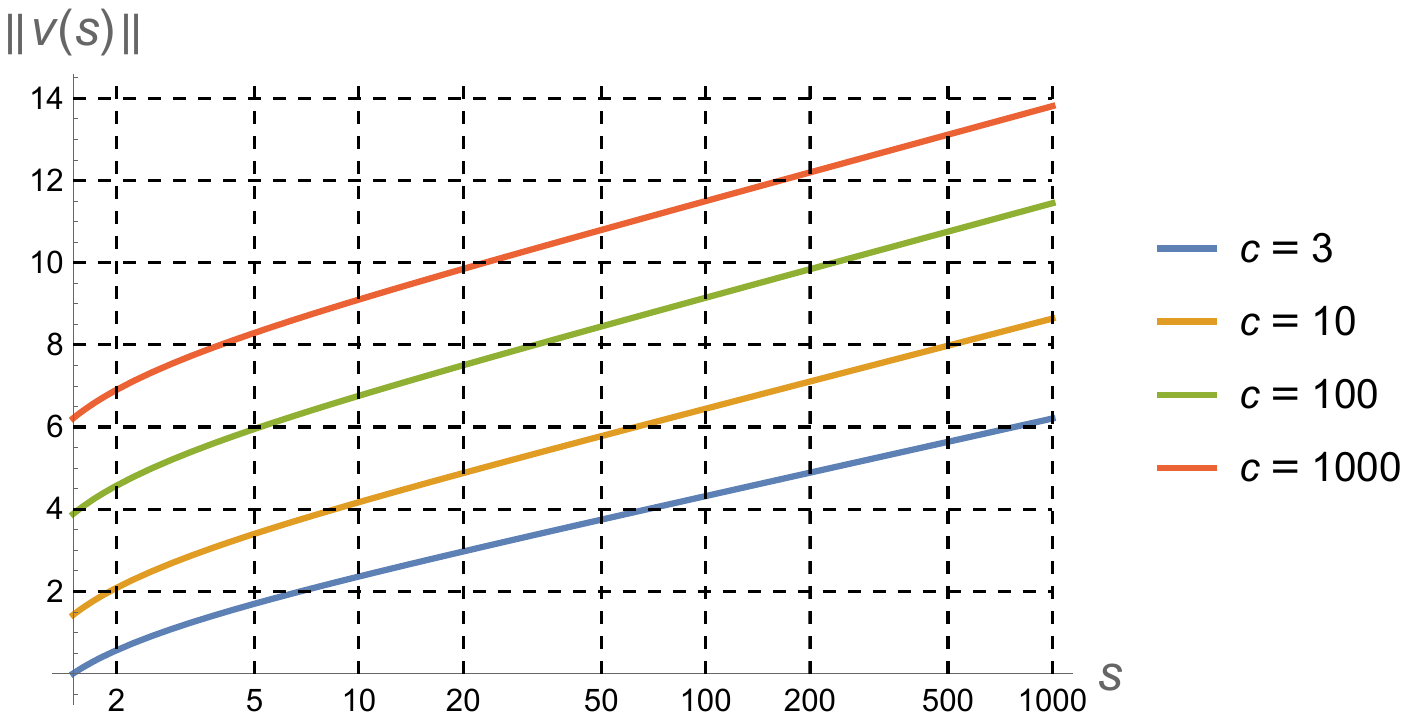}
}
\caption{Norms $\|v(s)\|$ of the tangent vectors $v(s) = \exp_{\eins_{\mc{S}}}^{-1}\big(p(s)\big)$ with $p(s)=(\frac{s-1}{s},\frac{1}{(c-1) s},\dotsc,\frac{1}{(c-1) s}\big) \to e_{1}\in\R^{c}$ if $s\to\infty$, for numbers of labels $c\in\{3,10,100,1000\}$. Since $\|e_{1}-p(s)\|=\big(\frac{c}{c-1}\big)^{1/2}\frac{1}{s}\approx \frac{1}{s}$, the simplex $\Delta_{c}$ is covered, up to a very small distance to its boundary, by $\exp_{\eins_{\mc{S}}}(B_{0}(r)) \subset \mc{S}_{c}$ and tangent vectors $v\in B_{0}(r)\subset T_{0}$ within a ball $B_{0}(r)$ centered at $0\in T_{0}$ with radius $r = 15$.
}
\label{fig:T0-norms}
\end{figure}

Second, by not moving all mass of the reference distribution (close) to $\ol{W}_\beta$ in finite time, we avoid a pathological behavior which can arise in flow matching on the simplex. 
Denote by
\begin{equation}\label{eq:rounding_region}
r_\beta = \Big\{W\in\mc{W}_{c}\colon \beta_i\in \arg\max_{j\in [c]}\,W_{i,j},\;\forall i\in [n]\Big\}
\end{equation}
the subset of points in $\mc{W}$ which assign their largest probability to the labels $\beta$.
\cite[Proposition~1]{Stark:2024aa} lays out that moving all mass of the reference distribution (close) to $\ol{W}_\beta$ in \textit{finite} time forces the model to make class decisions very early because the probability of $r_\beta$ under $\nu_t(\beta)$ increases too quickly. 
The effect is exacerbated by increasing the number of classes $c$ that the model is asked to discriminate.

However, by opting for large integration time $t\to\infty$ and a corresponding construction \eqref{eq:nu_cond_lifted_gauss} of conditional probability paths, our approach is able to scale to many classes $c\gg 1$, avoiding the pathology described in \cite[Proposition~1]{Stark:2024aa}.
Formally, this is because $\nu_t(\beta)$ defined in \eqref{eq:nu_cond_lifted_gauss} has full support on $\mc{W}_{c}$ for every $t\geq 0$. In practice, the parameter $\lambda$ in \eqref{eq:tangent_gaussian_path} can be used to control the speed at which the probability of $r_\beta$ under $\nu_t(\beta)$ increases, allowing the model to make class decisions gradually over time.

Figure \ref{fig:density_contour_timescales} (page \pageref{fig:density_contour_timescales}) displays probability density paths for illustration. 
The corresponding impact on model accuracy is quantitatively shown in Figure~\ref{fig:stark_simplex_scaling} (page \pageref{fig:stark_simplex_scaling}), with experimental details elaborated in Section~\ref{sec:class_scaling}.

\subsubsection{Relation to Dirichlet Flow Matching}\label{sec:comparison_to_stark}
The construction of \cite{Stark:2024aa} specifically addresses pathological behavior of flow matching on the simplex, by choosing conditional probability paths $\nu_t(\beta)$ as paths of Dirichlet distributions. They demonstrate that this approach scales to at least $c=160$ classes, by allowing the model to make class decisions gradually over time.
However, the explicit definition of $\nu_t(\beta)$ as paths of Dirichlet distributions makes it non-trivial to find corresponding vector fields $u_t(\cdot|\beta)$ for flow matching, which leads them to make an ansatz for fields which move mass along straight lines in the ambient Euclidean space in which the probability simplex is embedded. 

While we also make an explicit choice for $\nu_t(\beta)$ in \eqref{eq:nu_cond_lifted_gauss}, our construction is notably simpler than the approach of \cite{Stark:2024aa}, allowing to easily compute the vector fields $u_t(\cdot|\beta)$ by pushforward (Proposition \ref{prop:interpolants}).
The resulting flow moves mass along $e$-geodesics on $\mc{W}$, which is much more natural with respect to the information geometry of discrete probability distributions.
To illustrate this point, consider a straight path $\wh{p}(t)\in \R^n$ with direction $\frac{d}{dt} \wh{p}(t) = v\in \R^n$ at all times $t$. The trajectory $\wh{p}(t)$ is generated by maximizing $\la v, \wh{p}\ra$ along its gradient direction. On $\mc{W}_{c}$, the quantity $\la V_\beta, W\ra$ can be interpreted as correlation between $W\in \mc{W}_{c}$ and the direction $V_\beta$. The Riemannian gradient of this correlation with respect to the product Fisher-Rao geometry on $\mc{W}_{c}$ is $R_{W}[V_\beta]$, i.e.~precisely the direction of the conditional vector field \eqref{eq:condvectorfield-model}.


\subsection{Learning Interaction between Simplices}\label{sec:structured_flow_matching}
Our prior work \cite{Boll:2024aa} has studied the relationship between assignment flows on the product manifold $\mc{W}_{c}$ and replicator dynamics on the meta-simplex $\mc{S}_N$. 
We now use core results of \cite{Boll:2024aa} to derive the flow matching approach of Section~\ref{sec:learning} from \textit{first principles of flow matching in $\mc{S}_N$}, that is in the combinatorially large space of \textit{all} discrete joint distributions.
This demonstrates, in particular, that the proposed approach is suitable for \emph{structured prediction} settings, in which multiple \emph{coupled} random variables are of interest.

The result is surprising because \textit{direct} flow matching of joint distributions in $\mc{S}_N$ is \textit{intractable} due to the combinatorial dimension $N = c^n$. However, by leveraging the submanifold $\mc{T}$ (defined by \eqref{eq:TWc} and illustrated by Figure \ref{fig:Wright-AF}) and the compatibility of assignment flows with its geometry, we show that our construction can effectively break down combinatorial complexity and define a \textit{numerically tractable method}.

\vspace{0.2cm}
The map $T\colon \mc{W}_{c}\to \mc{S}_N$ defined in \eqref{eq:def-T-embedding} associates a marginal distribution of $n$ discrete random variables $W\in \mc{W}_{c}$ with a factorizing joint distribution $T(W)\in \mc{S}_N$.
Define with slight abuse of notation\footnote{$\pi_{0}$ is defined by \eqref{eq:def-pi0} as orthogonal projection onto the tangent space $T_{0}\mc{S}_{c}$ of the \textit{single} simplex $\mc{S}_{c}$ with trivial tangent bundle $\mc{S}_{c}\times T_{0}$. Here, to simplify notation, we overload $\pi_{0}$ to denote analogously the orthogonal projection onto the tangent space  $\mc{T}_{0}\mc{S}_{N}$.} the orthogonal projection 
\begin{equation}\label{eq:def-pi0-SN}
\pi_0\colon \R^N\to \mc{T}_0\mc{S}_N
\end{equation}
and formally denote the scaled standard normal distribution on $\mc{T}_0\mc{S}_N$ with variance $c^{n-1}$ by
\begin{equation}\label{eq:standard_normal_Sn_tangent}
    \mc{N}_0^{\mc{S}_N} = (\sqrt{c^{n-1}}\pi_0)_\sharp\mc{N}(0, I_{N}) = \mc{N}(0, c^{n-1}\pi_0\pi_0^\top) 
    = \mc{N}(0, c^{n-1}\pi_0).
\end{equation}
Analogous to the construction of conditional measures in Section~\ref{sec:conditional_fields}, we define the path of conditional measures
\begin{equation}\label{eq:Nt_cond_Sn_tangent}
    \mc{N}_t^{\mc{S}_N}(\cdot|\beta) = \mc{N}(\cdot; t c^{n-1}\lambda \pi_0e_\beta, c^{n-1}\pi_0)
\end{equation}
given a labeling $\beta\in [c]^n$ and a rate parameter $\lambda > 0$, scaled by the constant $c^{n-1}$.
It follows from Proposition~\ref{prop:cond_path_constraints} that 
\begin{equation}\label{eq:expNt_cond_Sn}
    \nu_t^{\mc{S}_N}(\beta) = (\exp_{\eins_{\mc{S}_N}})_\sharp \mc{N}_t^{\mc{S}_N}(\cdot|\beta)
\end{equation}
satisfies the conditions \eqref{eq:nu_conditional_endpoints} on $\mc{S}_N$ and is thus suitable for flow matching on $\mc{S}_N$ with reference distribution $\nu_{0}^{\mc{S}_N} = \mc{N}_0^{\mc{S}_N}$.
Formally, the Riemannian conditional flow matching criterion analogous to \eqref{eq:af_flow_matching} reads
\begin{equation}\label{eq:cond_fm_Sn}
    \mc{L}_{\mrm{RCFM}}^{\mc{S}_N}(\theta) = \EE_{t\sim\rho, \beta\sim p, q\sim \nu_t^{\mc{S}_N}(\beta)}\,\Big[\big\|R_q[\lambda \pi_0 e_\beta - f_\theta(q, t)]\big\|_w^2\Big]
\end{equation}
for an affinity function $f_\theta\colon \mc{S}_N\times [0,\infty)\to \mc{T}_0\mc{S}_N$. 

The task of minimizing \eqref{eq:cond_fm_Sn} is numerically intractable, because we are not even able to easily represent  general \emph{points} $q\in \mc{S}_N\setminus \mc{T}$ in the complement of the embedded assignment manifold $\mc{T}=T(\mc{W}_{c})$ given by \eqref{eq:def-T-embedding}.
To break down this complexity, we will define a projection onto $\mc{T}$ by using the \emph{lifting map lemma} \cite[Lemma~3.3]{Boll:2024aa}, which states
\begin{equation}\label{eq:lifting_map_lemma}
    \exp_{\eins_{\mc{S}_N}}(QV) = T\big(\exp_{\eins_{\mc{W}}}(V)\big)
\end{equation}
for all tangent vectors $V\in T_0\mc{W}$, with the mappings $T$ and $Q$ defined by \eqref{eq:def-T-embedding} and \eqref{eq:def-Q-embedding}. 
We start with an orthogonal projection $\mc{T}_0\mc{S}_N\to \mimg(Q)\cap \mc{T}_0\mc{S}_N$.

\begin{lemma}[\textbf{orthogonal projection onto $\mimg(Q)\cap \mc{T}_0\mc{S}_N$}]\label{lem:orth_projection}
The orthogonal projection $\proj_0$ of tangent vectors in $\mc{T}_0\mc{S}_N$ to the subspace $\mimg Q\cap T_0\mc{S}_N$ reads
\begin{equation}\label{eq:def_proj0}
\proj_0\colon \mc{T}_0\mc{S}_N\to \mimg Q\cap \mc{T}_0\mc{S}_N,
\qquad\qquad
    \proj_0(v) := {Q_c}\Pi_0 {Q_c}^\top, v\qquad \text{for }\;v\in \mc{T}_0\mc{S}_N,
\end{equation}
in terms of the linear operator 
\begin{equation}\label{eq:def-Qc}
{Q_c} := \frac{1}{\sqrt{c^{n-1}}}Q.
\end{equation}
\end{lemma}

Since \eqref{eq:lifting_map_lemma} ensures that $\exp_{\eins_{\mc{S}_N}}(\mimg Q)\subseteq \mc{T}$, we can now define the projection
\begin{equation}\label{eq:proj_pt_cond_to_T}
    \proj_{\mc{T}} := \exp_{\eins_{\mc{S}_N}}\circ \proj_0\circ \exp_{\eins_{\mc{S}_N}}^{-1}\colon \mc{S}_N\to \mc{T}.
\end{equation}
Under this projection, the conditional measures $\nu_t^{\mc{S}_N}(\beta)\in \mc{P}(\mc{S}_N)$ precisely induce the conditional probability paths $\nu_t(\beta)\in \mc{P}(\mc{W}_{c})$ defined by \eqref{eq:nu_cond_lifted_gauss}. 
Note that every extreme point of $\mc{S}_N$ lies in (the closure of) $\mc{T}$. Thus, projecting to $\mc{T}$ preserves the Dirac measures $\delta_{e_\beta}$ reached by the conditional distributions \eqref{eq:expNt_cond_Sn} in the limit $t\to \infty$.
In particular, the projection transforms the intractable conditional flow matching criterion \eqref{eq:cond_fm_Sn} on $\mc{S}_N$ into the numerically tractable criterion \eqref{eq:af_flow_matching}.
\begin{theorem}[\textbf{projected flow matching on $\mc{S}_N$}]\label{theorem:proj_Sn_fm}
For any $\beta\in [c]^n$, the pushforward of the conditional measure $\nu_t^{\mc{S}_N}(\beta)$ defined in \eqref{eq:expNt_cond_Sn} under the projection $\proj_{\mc{T}}\colon \mc{S}_N\to \mc{T}$ defined in \eqref{eq:proj_pt_cond_to_T} is
\begin{equation}\label{eq:proj_cond_measure}
(\proj_{\mc{T}})_\sharp \nu_t^{\mc{S}_N}(\beta)
    = T_\sharp \nu_t(\beta)
\end{equation}
with $\nu_t(\beta)\in \mc{P}(\mc{W}_{c})$ defined in \eqref{eq:nu_cond_lifted_gauss} and the embedding map $T$ given by \eqref{eq:def-T-embedding}.
Furthermore, the flow matching criterion on $\mc{T}$, induced by the conditional paths \eqref{eq:proj_cond_measure}, reads
\begin{equation}\label{eq:proj_flow_matching}
\mc{L}_{\mrm{RCFM}}^{\mc{T}}(\theta) = \EE_{t\sim\rho, \beta\sim p, q\sim (\proj_{\mc{T}})_\sharp \nu_t^{\mc{S}_N}(\beta)}\,\Big[\big\|R_q[\lambda \pi_0 e_\beta - \wt f_\theta(q, t)]\big\|_w^2\Big]
\end{equation}
and, using the ansatz $\wt f_\theta = Q \circ F_\theta \circ M$ with $Q$ and $M$ defined by \eqref{eq:def-Q-embedding} and \eqref{eq:def-M-map} , \eqref{eq:proj_flow_matching} is equal to the criterion \eqref{eq:af_flow_matching} for  matching assignment flows on $\mc{W}_{c}$.
\end{theorem}

Theorem~\ref{theorem:proj_Sn_fm} shows that the constructed flow matching on $\mc{W}_{c}$, which operates separately on multiple simplices, is \textit{induced} by flow matching in the \emph{single} meta-simplex $\mc{S}_N$, with conditional distribution paths and vector fields projected to the submanifold $\mc{T} = T(\mc{W}_{c})$.

This result provides a geometric justification of the fact that \textit{interaction} between simplices is learned through flow matching, even though all conditional probability paths $\nu_t(\beta)$ used for training can be \textit{separately} constructed on individual simplices.

\subsection{Numerical Flow Integration}
\label{sec:numerics}
We point out again that learning our generative model by Riemannian flow matching is 'simulation free': numerical integration is not required since only vector fields have to be matched which are defined on the tangent bundle of the assignment manifold and on the corresponding tangent-subspace distribution of the meta simplex (Prop.~\ref{theorem:proj_Sn_fm}), respectively. On the other hand, numerical integration of the flow is required for evaluating the learned generative model, in order to sample as illustrated by Figure \ref{fig:Wright-AF}, or for likelihood computation (Section \ref{sec:Likelihood}).

Since the flow corresponds to an ODE on a Riemannian manifold, \textit{geometric} numerical integration utilizes various representations of the ODE on the tangent bundle in order to apply established methods for numerical integration on Euclidean spaces \cite{Hairer:2006aa}. In the case of the assignment flow, this has been thoroughly studied by \cite{Zeilmann:2020aa} using the extension of the lifting map \eqref{eq:def-lifting-p} to the product manifold \eqref{eq:def-lifting-W}, regarded as action of the respective tangent space (regarded as additive abelian Lie group) on the assignment manifold. From the general viewpoint of geometric numerical integration, the resulting schemes for geometric numerical integration categorize as Runge-Kutta schemes of Munthe-Kaas type \cite{Munthe-Kaas:1999aa}.

Specifically, in this paper, numerical integration was carried out using the classical explicit embedded Dormand \& Prince Runge-Kutta method \cite{Dormand:1980aa} of orders 4 \& 5 with stepsize control (cf.~\cite[Section 5.2]{Zeilmann:2020aa} and \cite[Section II.5]{Hairer:2008aa}).

\subsection{Likelihood Computation}
\label{sec:Likelihood}
The likelihood of test data under the model distribution $\wt p$ is commonly used as a surrogate for Kullback-Leibler divergence between $\wt p$ and the true data distribution $p$,  due to the relationship
\begin{equation}\label{eq:kl_log_lh}
    \KL(p, \wt p) = \EE_p\Big[\log \frac{p}{\wt p}\Big] = -H(p) - \EE_p [\log \wt p].
\end{equation}
The entropy $H(p)$ is a property of the data distribution, which is not typically known, but can be treated as a constant which does not depend on the model used to approximate $\wt{p}$. 
For continuous normalizing flows, likelihood under the model is directly used as a training criterion, for this reason. Using the instantaneous change-of-variables formula \cite{Chen:2018}
\begin{equation}\label{eq:instant_change_of_variables}
\frac{\partial}{\partial t} \log \nu_t(x) = - \mtrace J(x, t),
\end{equation}
log-likelihood under continuous normalizing flows can, on continuous state spaces, be computed by integrating \eqref{eq:instant_change_of_variables} backward in time. In \eqref{eq:instant_change_of_variables}, $J(x,t)$ denotes the vector field Jacobian, whose trace is commonly approximated by using Hutchinson's estimator \cite{Hutchinson:1989}
\begin{equation}\label{eq:Hutchinson}
    \mtrace J = \EE_{v}[\la v, Jv\ra]
\end{equation}
with $v$ drawn from a fixed normal or Rademacher distribution. The use of this estimator in the context of likelihood under continuous normalizing flows was proposed by \cite{Grathwohl:2018}. 
The authors use a single sample $v$ for each integration of \eqref{eq:instant_change_of_variables}, which yields an unbiased estimator for log-likelihood of independent test data. 
In order to use likelihood as a training criterion, numerical integration of \eqref{eq:instant_change_of_variables} is required. This entails many forward and backward passes through the employed network architecture in order to compute a single parameter update.

Therefore, we do not use likelihood as a training criterion, opting instead for the simulation-free flow matching approach of Section~\ref{sec:learning}.
Since the learned model is still a normalizing flow, \eqref{eq:instant_change_of_variables} remains a useful tool for computing likelihoods under our model. 
However, because we are modeling discrete data while working on continuous state spaces, likelihood of discrete data can not be computed as a point estimate on $\mc{W}_{c}$. Further details are provided in Appendix \ref{sec:appendix-likelihood}.

\subsection{Dequantization}\label{sec:dequantization}
Approximation of discrete data distributions by continuous distributions has been studied through the lens of \emph{dequantization}. Choose a latent space $\mc{F}^n$ and an embedding of class label configurations $\beta\in [c]^n$ as prototypical points $f_\beta^\ast\in \mc{F}^n$. Suppose the choice of these points is fixed before training and associate disjoint sets $A_\beta\subseteq \mc{F}^n$ with label configurations such that they form a partition of $\mc{F}^n$ and $f_\beta^\ast\in A_\beta$. 
We can then define the continuous surrogate model 
\begin{equation}\label{eq:continuous_surrogate_theis}
    \vartheta = \sum_{\beta\in [c]^n} p_\beta \mc{U}_{A_\beta} \in \mc{P}(\mc{F}^n)
\end{equation}
which represents $p\in \mc{S}_N$ as a mixture of uniform distributions $\mc{U}_{A_\beta}$, supported on the disjoint subsets $A_\beta$. The underlying idea is that
\begin{equation}\label{eq:continuous_surrogate_region_prob}
    \PP_\vartheta(A_\beta) = \int_{A_\beta} \vartheta(y)dy = p_\beta \int_{A_\beta} \mc{U}_{A_\beta}(y) dy = p_\beta
\end{equation}
due to the disjoint support of mixture components in \eqref{eq:continuous_surrogate_theis}.
Denote a continuous model distribution by $\nu\in \mc{P}(\mc{F}^n)$. Using Jensen's inequality, we find
\begin{subequations}\label{eq:cont_surrogate_jensen_kl}
\begin{align}
-H(\vartheta) - \KL(\vartheta, \nu) &= \int \vartheta(y)\log \nu(y) dy
    = \sum_{\beta\in [c]^n} p_\beta \int_{A_\beta} \log \nu(y) dy \\
    &\leq \sum_{\beta\in [c]^n} p_\beta \log \int_{A_\beta} \nu(y) dy\\
    &= -H(p) - \KL(p, q)
\end{align}
\end{subequations}
for the discrete model distribution $q$ defined by
\begin{equation}\label{eq:dequantization_discrete_model}
    q_\beta = \int_{A_\beta} \nu(y) dy = \PP_{\nu}(A_\beta).
\end{equation}
Thus, fitting $\nu$ to $\vartheta$ by maximizing log-likelihood of smoothed data drawn from $\vartheta$ implicitly minimizes an upper bound on the relative entropy $\KL(p,q)$. In practice, drawing smoothed data from $\proposaldist$ amounts to adding noise to the prototypes $f_{\beta_k}^\ast\in \mc{F}^n$ of discrete data $\{\beta_k\}_{k\in [m]}$.

The above \textit{dequantization approach} was first proposed by \cite{Theis:2015}. 
Their reasoning justifies the previously known heuristic of adding noise to dequantize data \cite{Uria:2013}. It has thenceforth become common practice for training normalizing flows as generative models of images \cite{Dinh:2017, Salimans:2017} and was generalized to non-uniform noise distributions by \cite{Ho:2019}.
These authors focus on image data which, although originally continuous, are only available discretized into $8$-bit integer color values for efficient digital storage. 
In this case, the underlying continuous color imparts a natural structure on the set of discrete classes. Similar colors are naturally represented as prototypes which are close to each other with respect to some metric on the feature space $\mc{F}^n$.

For the \textit{general} discrete data considered here, such a structure is not available. 
As a remedy, \cite{Chen:2022aa} present an approach to learn the embedding jointly with likelihood maximization and defining the partition of $\mc{F}^n$ into subsets $A_\beta$ through Voronoi tesselation.
The rounding model variant \eqref{eq:likelihood_rounding_model} of our approach can be seen as dequantization on the space $\mc{F}^n=\mc{W}_{c}$ with prototypical points $f_{\beta}^\ast = \ol{W}_\beta$. The sets $A_\beta$ generated by Voronoi tesselation then coincide with the sets $r_\beta$ defined by \eqref{eq:rounding_region}.
However, our approach differs from \cite{Chen:2022aa} by using flow matching instead of likelihood-based training and by explicit consideration of information geometry on $\mc{W}_{c}$.

A natural question is whether the ability to learn an embedding of class configurations as prototypical points $f_\beta^\ast$, thereby representing similarity relations between classes, can be replicated in our setting.
Indeed, because points in $\mc{S}_c$ have a clear interpretation as categorical distributions, it is easy to achieve this goal by extending the affinity function $F_\theta$ of the assignment flow \eqref{eq:AF-general}.

For some $L > 0$, let $E\in \R^{L\times c}$ be a learnable embedding matrix. The columns of $E$ can be seen as prototypical points in the Euclidean latent space $\R^L$. The action of $E$ on an integer probability vector $e_j\in \mc{S}_c$ precisely selects one of these points, associating it with the class $j\in [c]$.
Learning $E$ now allows to represent relationships between classes in the latent space $\R^L$.
Let $\mc{E}\colon \R^{n\times c}\to \R^{n\times c}$ denote the linear operator which applies $E$ nodewise.
We now choose a parameterized function $\wt F_\theta\colon \R^L\to\R^L$ that operates on $\R^L$ and define the extended payoff function
\begin{equation}\label{eq:embedding_payoff_func}
    F_\theta = \mc{E}^\top \circ \wt F_\theta \circ \mc{E}\colon \mc{W}_{c}\to \R^{n\times c}.
\end{equation}

\section{Experiments and Discussion}\label{sec:Experiments}

As outlined in Section~\ref{sec:Approach}, we perform Riemannian flow-matching \eqref{eq:riemannian-flow-matching} via the conditional objective \eqref{eq:af_flow_matching} to learn assignment flows \eqref{eq:AF-general}. 
These in turn approximate $p_\infty$ in the limit $t\to\infty$ and thereby the unknown data distribution $p$ through \eqref{eq:marginal_limit_representation}.

\subsection{Class Scaling}\label{sec:class_scaling}
First, we replicate the experiment of \cite[Figure~4]{Stark:2024aa} to verify that our model is able to make decisions gradually over longer integration time and can scale to many classes $c$. Details of the training procedure are relegated to Appendix~\ref{sec:class_scaling_app}.
For each $c$, the data distribution is a randomly generated, factorizing distribution on $n=4$ simplices.

Figure~\ref{fig:stark_simplex_scaling} shows the relative entropy between the learned models (histogram of 512k samples) and the known target distribution. 
Our proposed approach is able to outperform our earlier method \cite{Boll:2024ab} (green) as well as Dirichlet flow matching \cite{Stark:2024aa} (orange) and the linear flow matching baseline (blue) in terms of scaling to many classes $c$.
In Figure~\ref{fig:stark_simplex_scaling}, the linear flow matching baseline scales better to many classes than in \cite[Figure~4]{Stark:2024aa}, but the qualitative statement that linear flow matching is ill-suited to this end is still supported by our empirical findings.
Our preliminary approach \cite{Boll:2024ab} (green) also scales comparatively well, even outperforming Dirichlet flow matching. 
Figures~\ref{fig:density_contour_timescales} and \ref{fig:density_rates-1} illustrate probability paths $\nu_t(\beta)$ for our approach (cf.~\eqref{eq:nu_cond_lifted_gauss}) and Dirichlet flow matching \cite{Stark:2024aa} at different time scales. 

A property of assignment flow approaches, possibly linked to observed performance, is to transport probability mass relative to the underlying Fisher Rao geometry (recall Section~\ref{sec:comparison_to_stark}). For example, this leads to little probability mass in regions close to the simplex boundaries (Figure~\ref{fig:density_contour_timescales}).

\begin{figure}[ht]
\centering
\includegraphics[width=0.5\textwidth]{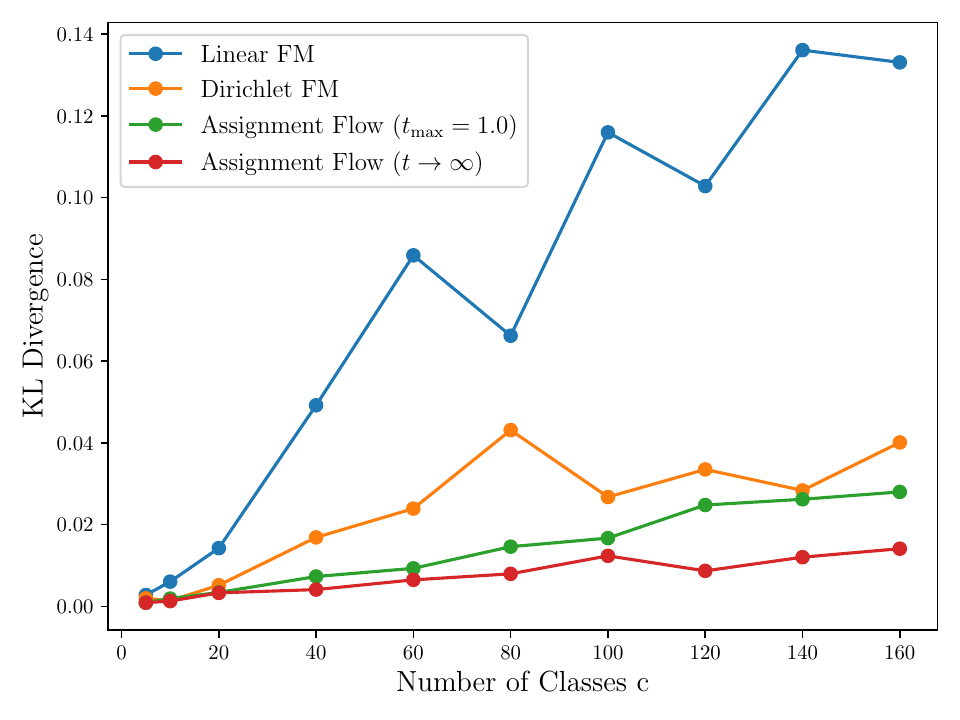}
\caption{Relative entropy between learned models (histogram of 512k samples) and a known, factorizing target distribution on $n=4$ simplices with varying number of classes $c$. By leveraging information geometry and gradual decision-making over time, our proposed approach (red) is able to outperform our earlier method \cite{Boll:2024ab} as well as Dirichlet flow matching \cite{Stark:2024aa} in terms of scaling to many classes $c$.}\label{fig:stark_simplex_scaling}
\end{figure}

\subsection{Generating Image Segmentations}\label{sec:experiments_img}
In image segmentation, a joint assignment of classes to pixels is usually sought conditioned on the pixel values themselves. Here, we instead focus on the \emph{unconditional} discrete distribution of segmentations, without regard to the original pixel data. These discrete distributions are very high-dimensional in general, with $N = c^n$ increasing exponentially in the number of pixels.

To this end, we parametrize $F_\theta$ by the UNet architecture of \cite{Dhariwal2021diffusion} and train on downsampled segmentations of Cityscapes \cite{Cordts2016Cityscapes} images, as well as MNIST \cite{lecun2010mnist}, regarded as binary $c=2$ segmentations after thresholding continuous pixel values at $0.5$.
Details of the training procedure are relegated to Appendix~\ref{sec:experiments_img_app}.

Figures~\ref{fig:mnist_similarity} and~\ref{fig:city_similarity} show samples from the learned distribution of binarized MNIST and Cityscapes segmentations respectively, next to the closest training data. This illustrates that our model is able to interpolate the data distribution, without simply memorizing training data.
Additional samples from our Cityscapes model are shown in Figure~\ref{fig:city_samples}.

\begin{figure}[hb]
    \centering
    \includegraphics[width=0.7\textwidth]{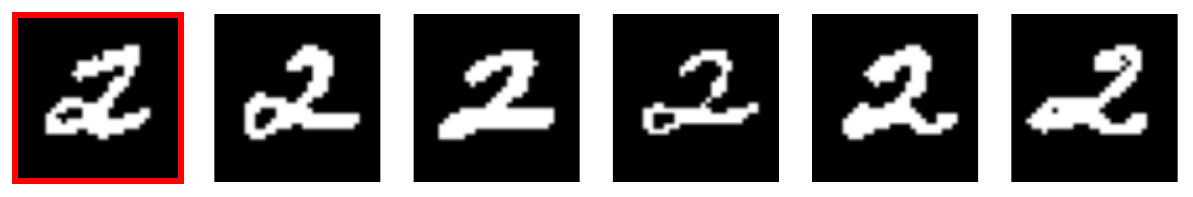}
    \includegraphics[width=0.7\textwidth]{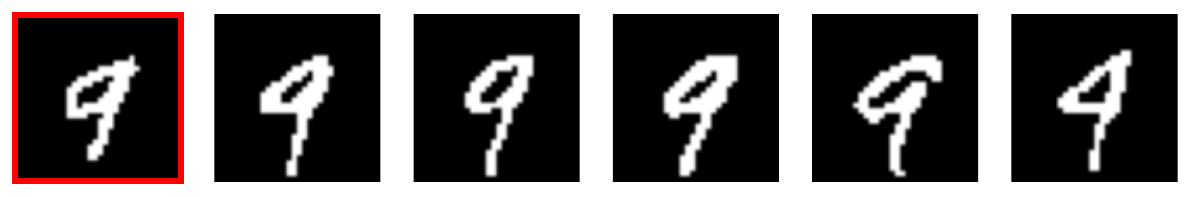}
    \includegraphics[width=0.7\textwidth]{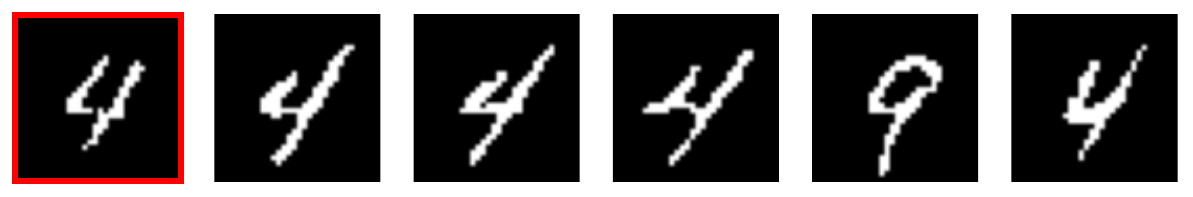}
    \includegraphics[width=0.7\textwidth]{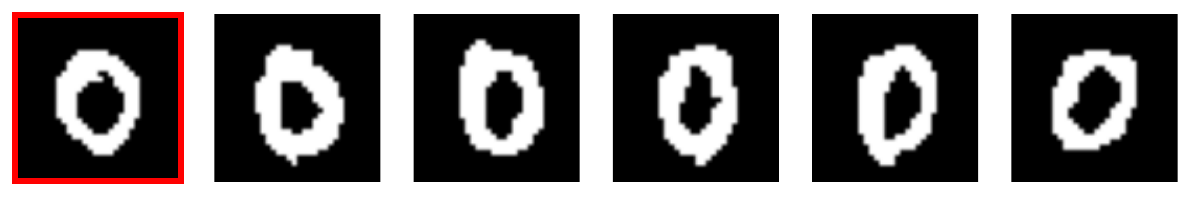}
    \includegraphics[width=0.7\textwidth]{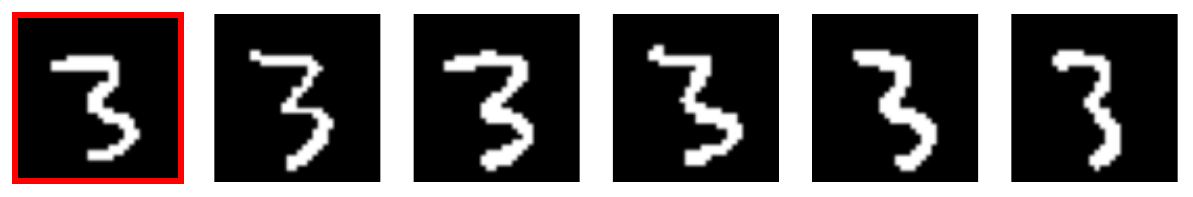}
    \caption{Comparison of model samples to the closest training data. \emph{Left with red border}: samples drawn from our model of the binarized MNIST distribution. \emph{Right}: training data closest to the sample in terms of pixel-wise distance.}
    \label{fig:mnist_similarity}
\end{figure}

\begin{figure}[htbp]
    \centering
    \includegraphics[width=1.\textwidth]{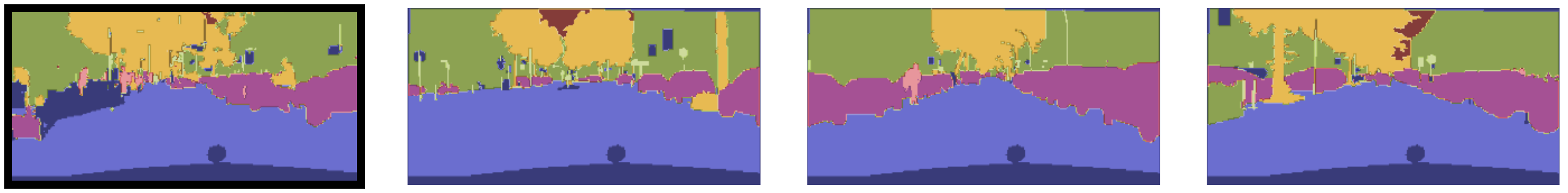}
    \includegraphics[width=1.\textwidth]{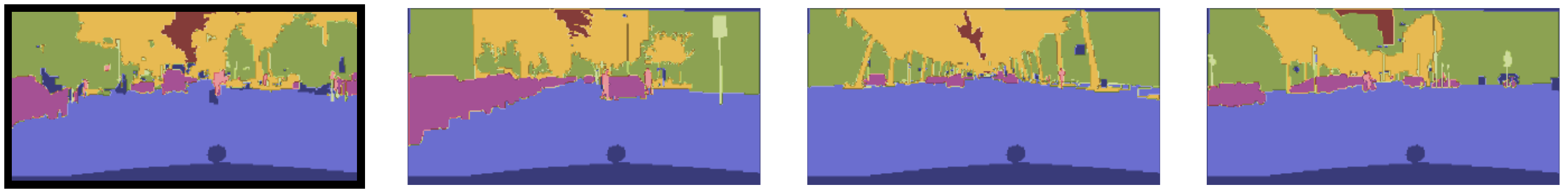}
    \includegraphics[width=1.\textwidth]{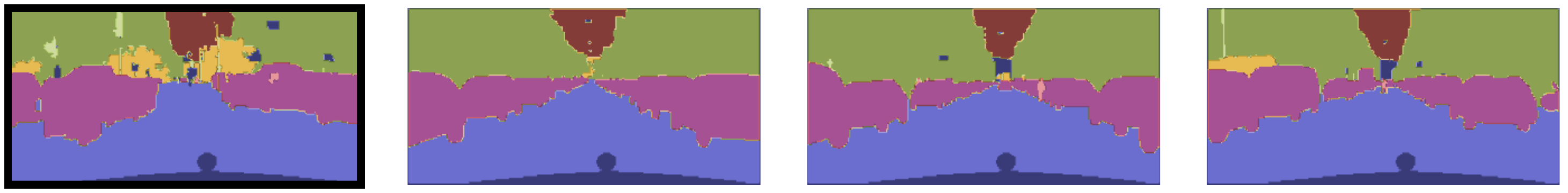}
    \includegraphics[width=1.\textwidth]{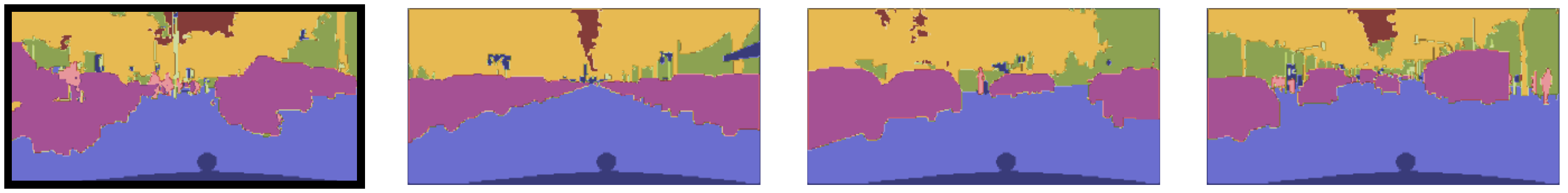}
    \includegraphics[width=1.\textwidth]{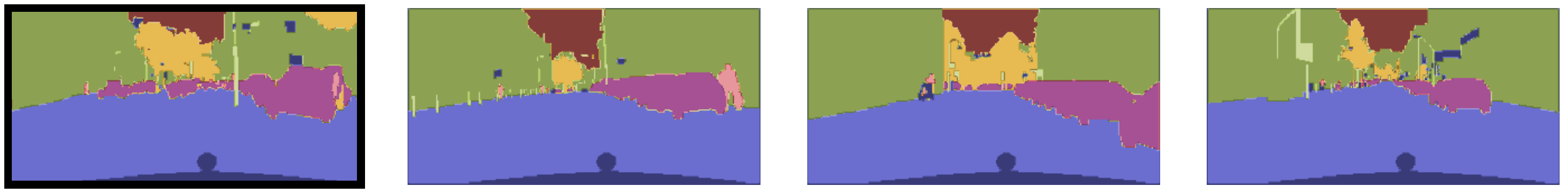}
    \caption{Comparison of model samples to the closest training data. \emph{Left with black border}: samples drawn from our model of the Cityscapes segmentation distribution. \emph{Right}: training data closest to the sample in terms of pixel-wise distance.}
    \label{fig:city_similarity}
\end{figure}

\begin{figure}[htbp]
    \centering
    \includegraphics[width=1.\textwidth]{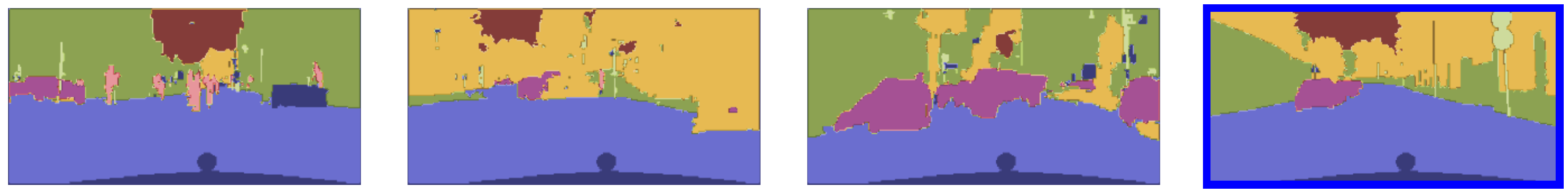}
    \includegraphics[width=1.\textwidth]{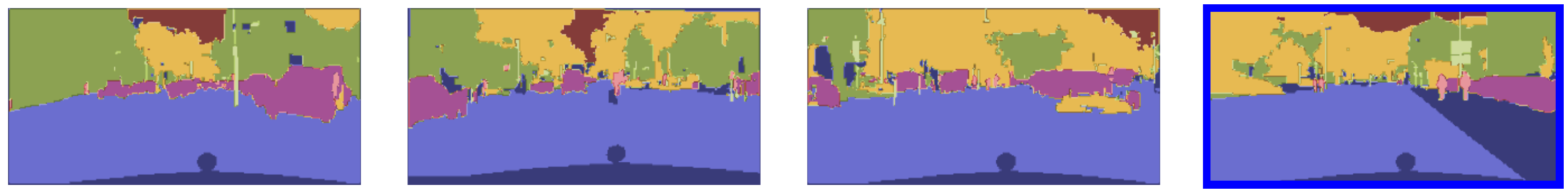}
    \includegraphics[width=1.\textwidth]{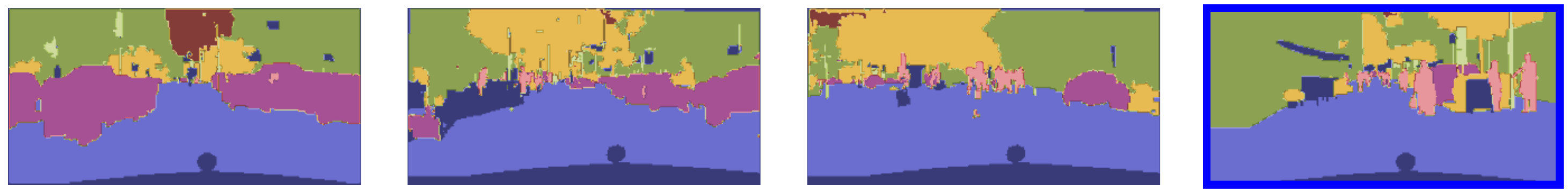}
    \includegraphics[width=1.\textwidth]{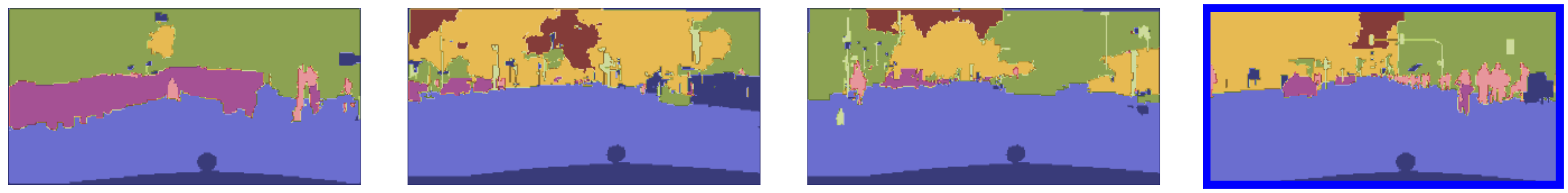}
    \includegraphics[width=1.\textwidth]{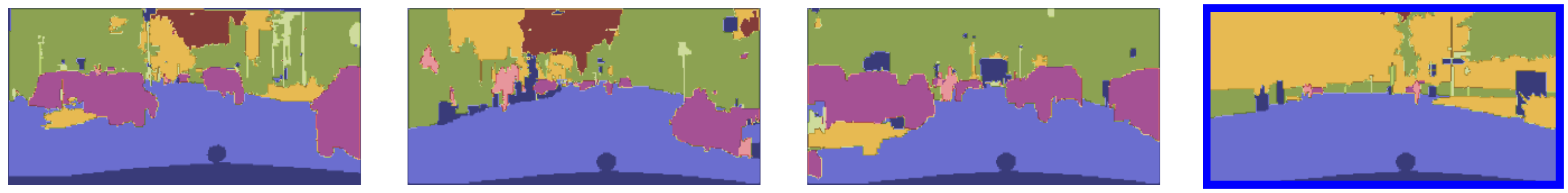}
    \caption{\emph{Left}: Samples from our model of the Cityscapes segmentation distribution. \emph{Right with blue border}: randomly drawn training data.}
    \label{fig:city_samples}
\end{figure}

\begin{figure}[htbp]
    \centering
    \includegraphics[width=0.95\textwidth]{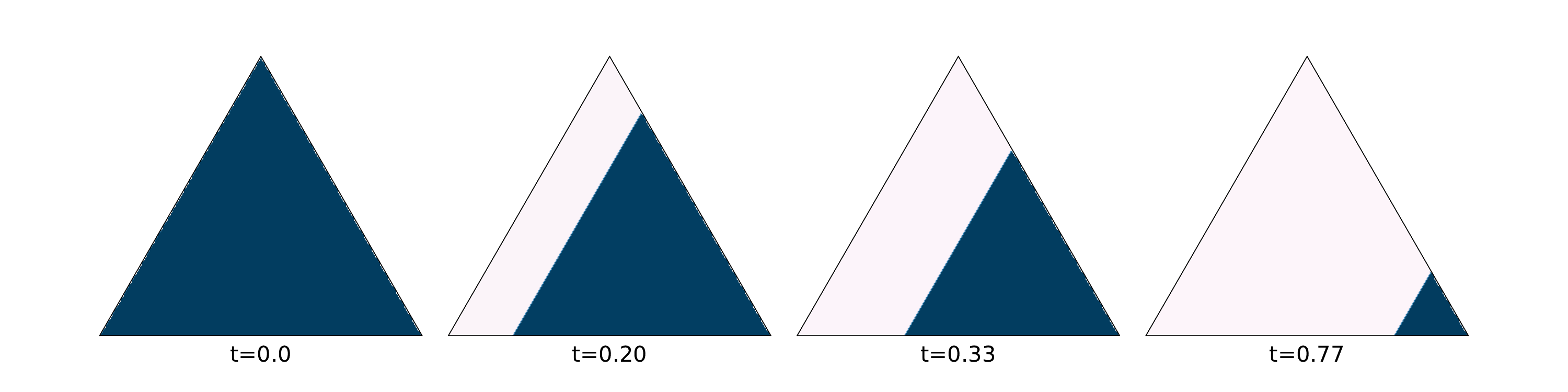}
    \includegraphics[width=0.95\textwidth]{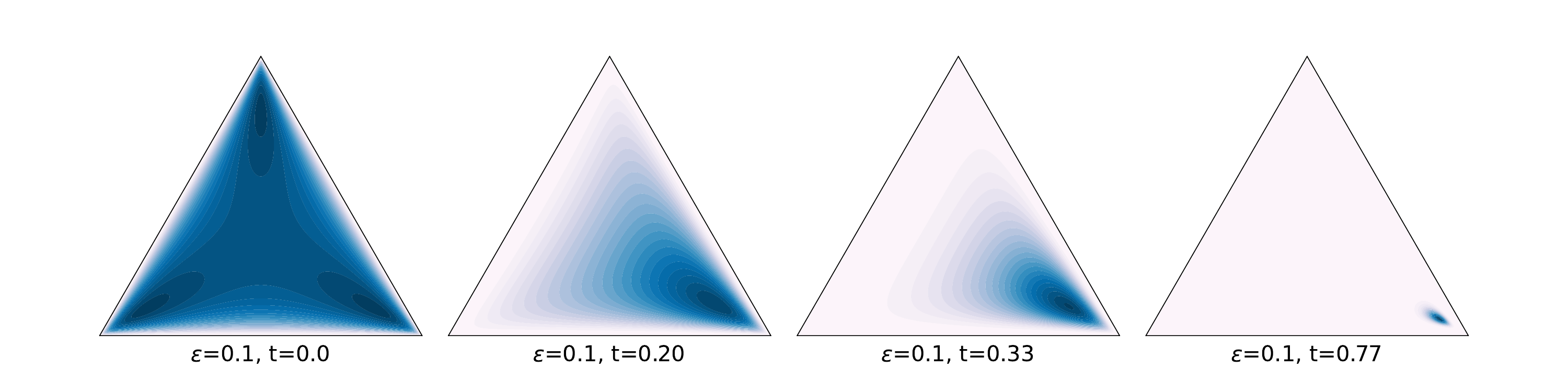}
    \includegraphics[width=0.95\textwidth]{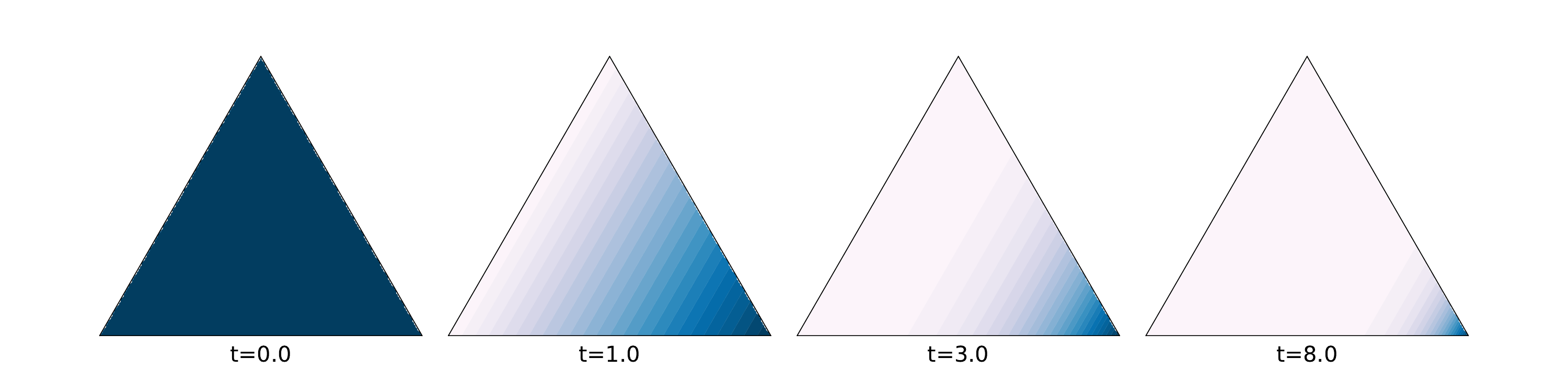}
    \includegraphics[width=0.95\textwidth]{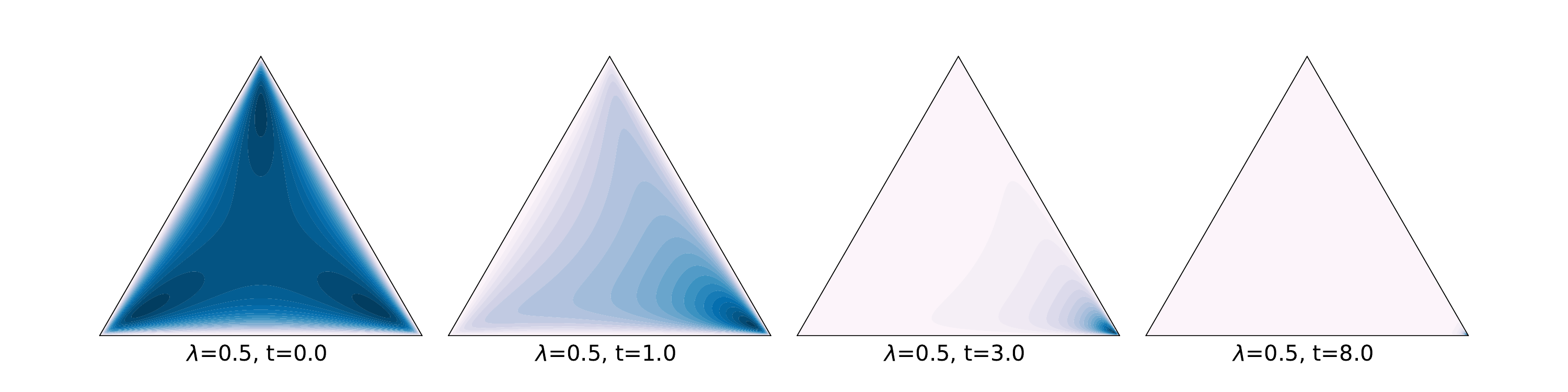}
    \includegraphics[width=0.95\textwidth]{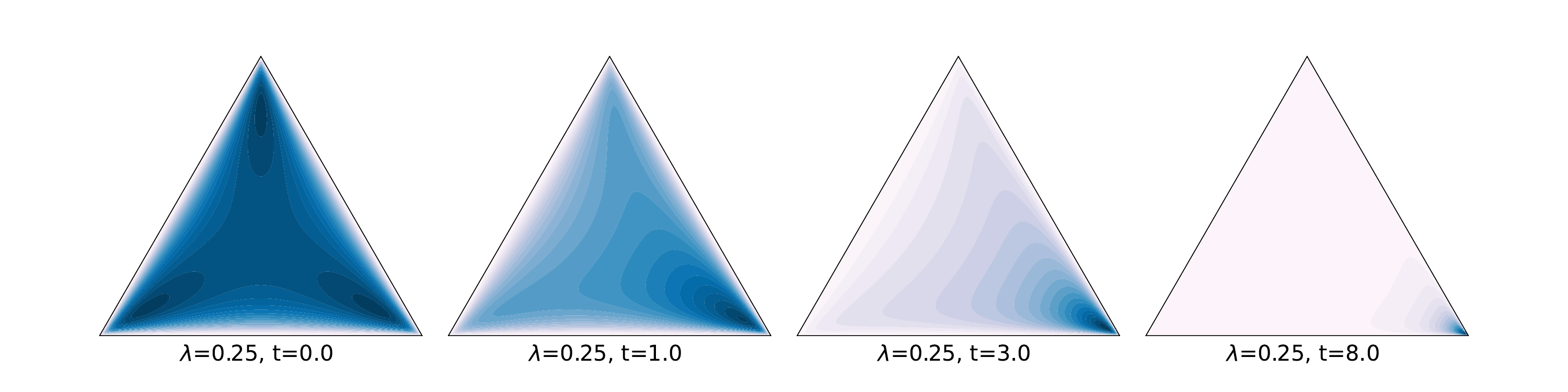}
\caption{Plots of conditional densitities $\nu_{t}(\beta)$ for different points of time $t$. Darker colors indicate higher concentration within the densities. \textit{From top to bottom:} Linear Flow Matching \cite[Equation 11]{Stark:2024aa}, the approach \cite[Equation 18]{Boll:2024ab}, Dirichlet Flow Matching \cite[Equation 14]{Stark:2024aa}, our approach \eqref{eq:nu_cond_lifted_gauss} using two different values of the rate parameter $\lambda$. Note the different time periods $t\in[0, 0.77]$ used for the first two and $t\in [0,8]$ for the latter approaches. See Section \ref{sec:class_scaling} for a discussion.}
    \label{fig:density_contour_timescales}
\end{figure}

\begin{figure}[htbp]
    \centering
    \includegraphics[width=1.\textwidth]{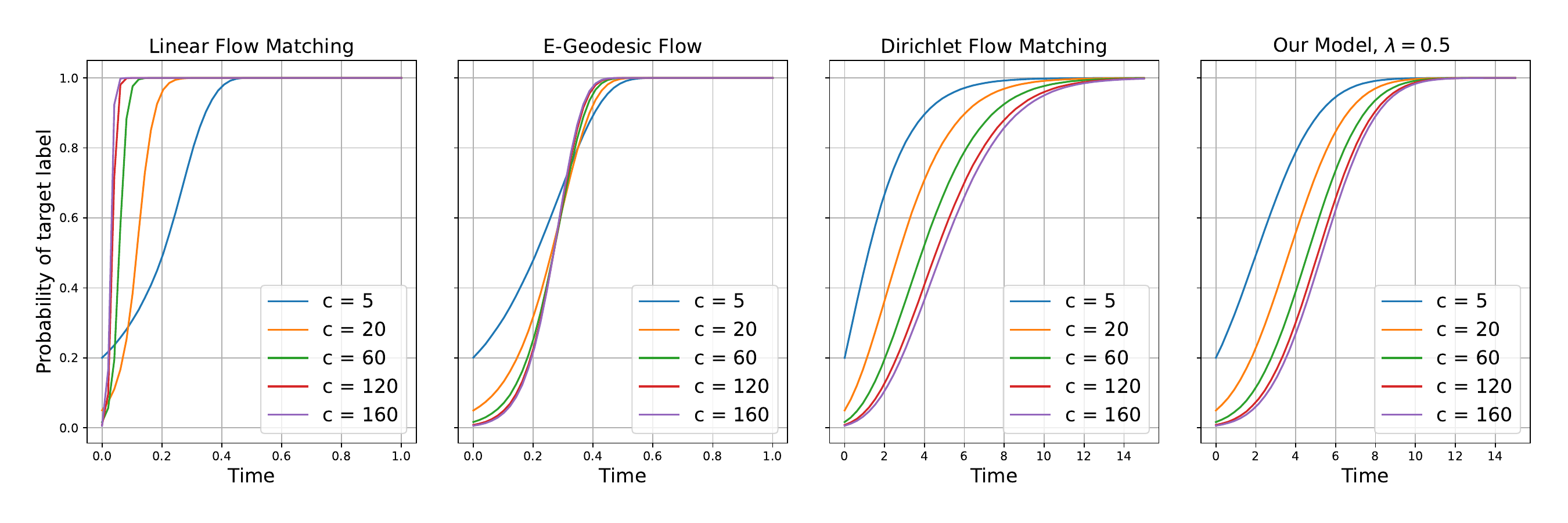}
    \includegraphics[width=1.\textwidth]{Figures/rate_plots_comparison_difflambda.pdf}
\caption{\textit{Top row:} Plots of conditional densitities paths $t\mapsto \nu_{t}(\beta)$ for various models. \textit{Bottom row:} Impact of the rate parameter $\lambda$ of our approach (replication of Figure \ref{fig:rate-parameter-lambda} to ease visual comparison).}
    \label{fig:density_rates-1}
\end{figure}

\subsection{Likelihood Evaluation}\label{sec:likelihood_experiments}
We compute the likelihood of test data from the MNIST dataset (binarized by thresholding) using the method described in Section~\ref{sec:Likelihood}.
We use 100 priority samples per datum and, as is common practice for normalizing flows, only a single Hutchinson sample. 
The result is shown in Table~\ref{tab:likelihood}, compared to our earlier approach \cite{Boll:2024ab} ($t\to 1$). For comparison, we show likelihood of MNIST test data (from the \emph{continuous}, non-binarized distribution) under several normalizing flow methods from the literature which were trained using likelihood maximization.

Note that, although much prior work on generative modelling has been applied to continuous gray value MNIST images, binarization (in our case through thresholding) substantially changes the data distribution. Thus, likelihood of test data, which is commonly used as a surrogate for relative entropy to the data distribution in normalizing flows, is not comparable between these methods and ours.
In addition, since we do not use likelihood maximization as a training criterion, it is not to be expected that our model is competitive on this measure. 
Still, the results of Table~\ref{tab:likelihood} indicate that the proposed model ($t\to \infty$) fits the binarized MNIST data distribution better in terms of relative entropy than our previous approach \cite{Boll:2024ab} ($t\to 1$).

\begin{table}[ht]
\caption{Likelihood of binarized MNIST test data under our proposed model ($t\to \infty$) and the earlier version \cite{Boll:2024ab} ($t\to 1$). Both methods are trained by flow matching rather than likelihood maximization.}\label{tab:likelihood}
\begin{tabular}{@{}lll@{}}
\toprule
Method                  & AF ($t\to \infty$) & AF ($t\to 1$)   \\ \midrule
Likelihood (bits / dim) & 1.01 $\pm$ 0.17    & 4.05 $\pm$ 0.83 \\ \bottomrule
\end{tabular}
\end{table}


\newpage
\section{Conclusion}\label{sec:Conclusion}

We introduced a novel generative model for the representation and evaluation of joint probability distributions of discrete random variables. The approach employs an embedding of the assignment manifold in the meta-simplex of all joint probability distributions. Corresponding measure transport by randomized assignment flows approximates joint distributions of discrete random variables in a principled manner. The approach enables to learn the statistical dependencies of any set of discrete random variables and using the resulting model for structured prediction, independent of the area of application.

Inference using the approach is computationally efficient, since sampling can be accomplished by parallel geometric numerical integration. Training the generative model using given empirical data is computationally efficient, since matching the flow of corresponding e-geodesics is used as training criterion, which does not require sampling as a subroutine. 

Numerical experiments showed superior performance in comparison to recent related work, which we attribute to consistently using the underlying information geometry of assignment flows and the corresponding measure transport along conditional probability paths. On the other hand, 
the fact that even our \textit{preliminary} approach \cite{Boll:2024ab} can outperform Dirichlet flow matching \cite{Stark:2024aa} with respect to scaling to many classes in Figure~\ref{fig:stark_simplex_scaling}, is surprising, because the approach \cite{Boll:2024ab} uses a \textit{finite} integration time and moves all mass of the reference distribution to a Dirac measure close to $\ol{W}_\beta$ within this finite time. The core assumptions of \cite[Proposition~1]{Stark:2024aa}, therefore, apply to this approach, and the fact that it still performs well empirically suggests that further inquiry into this topic is warranted.

\newpage
\appendix 
\section{Proofs}

\subsection{Proofs of Section \ref{sec:conditional_fields}}\label{sec:proofs-conditional-fields}

\begin{proof}[Proof of Proposition \ref{prop:interpolants}]
Since $V_{\beta}$ is determined by $\beta$ and does not depend on $V$, the map $V\mapsto V + \lambda t \lambda V_{\beta} $ is affine. Hence, Eq.~\eqref{eq:condflow-model} conforms to
\eqref{eq:nu_cond_lifted_gauss}, because affine transformations of normal distributions are again normal distributions. The mapping $\exp_{\eins_{\mc W}}(\cdot) : \mc W \to \mc T_0$ is a diffeomorphism. Consequently, the inverse of \eqref{eq:condflow-model} can be computed from
\begin{subequations}\label{eq:inv-interpolant}
\begin{align}
& W :=\psi_{t}(V|\beta) = \exp_{\eins_{\mc W}} \big( V + t\lambda V_\beta \big) \label{eq:inv-interpolant-1}  
\\ \label{eq:inv-psi-t}
    \Leftrightarrow  \qquad & \psi_t^{-1}(W|\beta) = V = \exp_{\eins_{\mc W}}^{-1}(W)- t\lambda V_\beta ,
\end{align}
\end{subequations}
which verifies \eqref{eq:condflow-inv-model}. Regarding \eqref{eq:condvectorfield-model}, recall that the conditional flow is determined by the conditional vector field through the ODE 
\begin{equation}\label{eq:cond-flow}
    \frac{d}{dt} \psi_t(V|\beta) = u_t\big(\psi_t(V|\beta) \big| \beta\big), \qquad  \psi_0(V|\beta) = \psi_{0}(V)=\exp_{\eins_{\mc{W}}}(V).
\end{equation}
On the other hand, direct computation of the time derivative of  \eqref{eq:inv-interpolant-1} using the closed-form expression
\begin{equation}
d\exp_{W}(V) [U] = R_{\exp_{W}(V)}[U]
\end{equation}
for the differential of the lifting map \eqref{eq:def-lifting-W}, yields
\begin{equation}\label{eq:diff-interpolant}
    \frac{d}{dt} \psi_t(V|\beta) 
    = R_{\psi_t(V|\beta)}[\lambda V_{\beta}].
\end{equation}
Equating \eqref{eq:cond-flow} and \eqref{eq:diff-interpolant} and using $W=\psi_{t}(V|\beta)$  from \eqref{eq:inv-psi-t} proves
\eqref{eq:condvectorfield-model}.

\end{proof}

\begin{proof}[Proof of Proposition 
\ref{prop:cond_path_constraints}]\label{proof-3.3}
Equation \eqref{eq:nu0-conditional} is immediate due to \eqref{eq:def-mcN0}, \eqref{eq:def-nu0} and \eqref{eq:nu_cond_lifted_gauss}. Writing short
\begin{equation}
\psi_{t}:= \psi_{t}(\cdot|\beta)
\end{equation}
for the flow map defined by \eqref{eq:condflow-model}, it remains to show that
\begin{equation}\label{eq:toshow}
    \lim_{t \to \infty} \nu_{t}(\beta) = \lim_{t\to \infty} (\psi_{t})_{\sharp} \nu_0 = \delta_{\ol W_\beta}.
\end{equation}
To this end, we demonstrate that every marginal of the conditional probability path \eqref{eq:toshow} converges to a Dirac measure supported on the assignment vector corresponding to the labeling configuration $\beta$, i.e. 
\begin{equation}\label{eq:toshow-marginal}
    \lim_{t \to \infty} \nu_{t;i}(\beta)  
    = \lim_{t\to \infty}(\psi_{t;i})_{\sharp}\nu_{0;i} =\delta_{\ol W_{\beta;i}}, \quad i\in[n],
\end{equation}
where $\nu_{0;i},\,i\in[n]$ denote the marginals of $\nu_0$ given by \eqref{eq:def-nu0}. 

First, we observe that by fixing an orthonormal basis of $T_0$ as column vectors of the matrix $\mc B$, every marginal $\nu_{0;i}$ of \eqref{eq:def-nu0} with Gaussian $\mc N_0$ defined by \eqref{eq:def-mcN0} can be expressed as the lifted image measure of a standard normal distribution $\mc N\big(0_{c-1},I_{c-1}\big)$ on $\R^{c-1}$ with respect to the basis $\mc B$, 
\begin{equation}
\nu_{0;i} = (\exp_{\eins_{\mc{S}}})_{\sharp}{\mc B}_{\sharp}\mc N\big(\cdot; 0_{c-1}, I_{c-1}\big) 
= (\exp_{\eins_{\mc{S}}})_{\sharp}\mc{N}(\cdot; 0_{c},\pi_{0}),
\end{equation}
since $\mc B {\mc B}^\T = \pi_{0}$. Consequently, by Proposition \ref{prop:interpolants},
\begin{equation}
\nu_{t;i}(\beta) 
=\big(\psi_{t;i}\big)_{\sharp}\mc{N}(\cdot; 0_{c},\pi_{0})
\end{equation}
and hence using the change-of-variables formula and \eqref{eq:inv-psi-t}, one has for any $p\in\mc{S}_{c}$,
\begin{equation}
\nu_{t;i}(p|\beta) = \mc{N}\big(\exp_{\eins_{\mc S}}^{-1}(p) - t \lambda V_{\beta;i}; 0_{c},\pi_{0}\big) |\det d\psi_{t;i}^{-1}|.
\end{equation}
Equation \eqref{eq:condflow-inv-model} shows that the differential $d\psi_{t;i}^{-1}$ does not depend on $t$. Neither does the normalizing factor of the normal distribution, due to the covariance matrix $\pi_{0} = \mathrm{id}_{T_{0}}$. Consequently, since $\psi_{t;i}^{-1}$ maps to $T_{0}$,
\begin{subequations}
\begin{align}
    \nu_{t}( p | \beta)  
    &\propto  \exp\Big(-\frac{1}{2}\big\la \exp_{\eins_{\mc S}}^{-1}(p) -t \lambda V_{\beta;i}, \pi_{0} \big(\exp_{\eins_{\mc S}}^{-1}(p) -t \lambda V_{\beta;i} \big) \big\ra\Big) \\
    &= \exp\Big(-\frac{1}{2}\big\la \exp_{\eins_{\mc S}}^{-1}(p) -t \lambda V_{\beta;i}, \big(\exp_{\eins_{\mc S}}^{-1}(p) -t \lambda V_{\beta;i} \big) \big\ra\Big) \to 0 \quad\text{as}\quad t\to \infty,
\end{align}
\end{subequations}
for any $p\neq \ol{W}_{\beta;i}\in\ol{\mc{S}_{c}}$ and $i\in[n]$, due to the choice \eqref{eq:tangent_gaussian_path} of the tangent vector $V_{\beta}$. 
We conclude that the image measure $\nu_{\infty;i}(\beta)$ is a Dirac measure concentrated on $\ol W_{\beta;i}$.
\end{proof}

\newpage
\subsection{Proofs of Section \ref{sec:structured_flow_matching}}

\begin{proof}[\text{Proof of Lemma \ref{lem:orth_projection}}]
By \cite[Lemma~4]{Boll:2021vb}, one has $Q^\top QV = c^{n-1}V$ for all $V\in \mc{T}_{0}$. Thus, ${Q_c}$ defined by \eqref{eq:def-Qc} has the property
\begin{equation}\label{eq:Q_normalized}
{Q_c}^\top{Q_c} V = V,\qquad \text{for all }\; V\in \mc{T}_{0}.
\end{equation}
To show that \eqref{eq:def_proj0} indeed defines the orthogonal projection onto $\mimg Q\cap \mc{T}_0\mc{S}_N$, note that 
\begin{equation}\label{eq:Qc-Pi0-relation}
{Q_c}\Pi_0 = \pi_0{Q_c}
\end{equation}
by \cite[Lemma~A.3]{Boll:2024aa} and accordingly
\begin{equation}\label{eq:Qt_Pi0_commute}
{Q_c}^\top\pi_0 
    = (\pi_0{Q_c})^\top 
    = ({Q_c}\Pi_0)^\top 
    = \Pi_0 {Q_c}^\top
\end{equation}
by using the symmetry of $\Pi_0$ and $\pi_0$. We can use this to show $\mimg \proj_0\subseteq \mimg Q\cap \mc{T}_0\mc{S}_N$, because for any $x\in \R^{n\times c}$, we have
\begin{equation}\label{eq:proj_0_subset_img}
{Q_c}\Pi_0 x \in \mimg Q\qquad\text{and}\qquad {Q_c}\Pi_0x 
\overset{\eqref{eq:Qc-Pi0-relation}}{=} 
\pi_0{Q_c}x \in \mc{T}_0\mc{S}_N.
\end{equation}
Now let $v\in \mc{T}_0\mc{S}_N$ and $y\in \mimg Q\cap \mc{T}_0\mc{S}_N$ be arbitrary. Then $y$ can be written as $y = {Q_c}y'$ and we have
\begin{subequations}\label{eq:proj_0_proof}
\begin{align}
\la v - \proj_0(v), y\ra &= \la v - {Q_c}\Pi_0{Q_c}^\top v, {Q_c}y'\ra
    = \la {Q_c}^\top v - {Q_c}^\top{Q_c}\Pi_0{Q_c}^\top v, y'\ra
    \\
    &\overset{\eqref{eq:Q_normalized}}{=}
    \la {Q_c}^\top v - \Pi_0{Q_c}^\top v, y'\ra
    \overset{\eqref{eq:Qt_Pi0_commute}}{=}
    \la {Q_c}^\top v - {Q_c}^\top \pi_0 v, y'\ra\\
    &= 0,
\end{align}
\end{subequations} 
which shows that $\proj_0$ projects orthogonally.
\end{proof}

\begin{proof}[\text{Proof of Theorem \ref{theorem:proj_Sn_fm}}]
We use the representation of $\proj_0$ (Lemma \ref{lem:orth_projection}) to compute the pushforward \eqref{eq:proj_cond_measure}. 
\begin{subequations}\label{eq:proj_cond_measure_proof}
\begin{align}
(\proj_{\mc{T}})_\sharp \nu_t^{\mc{S}_N}(\beta)
    &\overset{\eqref{eq:proj_pt_cond_to_T}}{=} (\exp_{\eins_{\mc{S}_N}}\circ \proj_0\circ \exp_{\eins_{\mc{S}_N}}^{-1})_\sharp \nu_t^{\mc{S}_N}(\beta)\\
    &\stackrel{\eqref{eq:expNt_cond_Sn}}{=} (\exp_{\eins_{\mc{S}_N}}\circ \proj_0)_\sharp \mc{N}_t^{\mc{S}_N}(\cdot|\beta)\\
    &\stackrel{\eqref{eq:def_proj0}}{=} (\exp_{\eins_{\mc{S}_N}}\circ {Q_c}\Pi_0 {Q_c}^\top)_\sharp \mc{N}_t^{\mc{S}_N}(\cdot|\beta)\\
        &\stackrel{\eqref{eq:Nt_cond_Sn_tangent}}{=} (\exp_{\eins_{\mc{S}_N}}\circ {Q_c}\Pi_0 {Q_c}^\top)_\sharp \mc{N}(\cdot;t c^{n-1}\lambda \pi_0e_\beta, c^{n-1}\pi_0)\\
    &= (\exp_{\eins_{\mc{S}_N}})_\sharp \mc{N}(\cdot; t  c^{n-1}\lambda {Q_c}\Pi_0 {Q_c}^\top \pi_0e_\beta, c^{n-1}{Q_c}\Pi_0 {Q_c}^\top\pi_0({Q_c}\Pi_0 {Q_c}^\top)^\top)\\
&\overset{\substack{\eqref{eq:def-Qc} \\ \eqref{eq:Qt_Pi0_commute}}}{=}
(\exp_{\eins_{\mc{S}_N}})_\sharp \mc{N}(\cdot; t \lambda Q\Pi_0 Q^\top e_\beta, c^{n-1}{Q_c}\Pi_0 {Q_c}^\top {Q_c}\Pi_0 {Q_c}^\top)\\
&\overset{\substack{\eqref{eq:def-Qc} \\\eqref{eq:Q_normalized}}}{=}
(\exp_{\eins_{\mc{S}_N}})_\sharp \mc{N}(\cdot; t\lambda Q \Pi_0 Q^\top e_\beta, Q \Pi_0Q^\top)
\\
&= (\exp_{\eins_{\mc{S}_N}} \circ Q)_\sharp \mc{N}(\cdot; t\lambda \Pi_0 Q^\top e_\beta, \Pi_0).
\end{align}
\end{subequations}
By \cite[Lemma~3.4]{Boll:2024aa}, we have $Q^\top e_\beta = M e_\beta$, with $Q$ and $M$ defined by \eqref{eq:def-Q-embedding} and \eqref{eq:def-M-map}. Using the shorthand $V_\beta$ defined by \eqref{eq:tangent_gaussian_path} and the lifting map lemma \eqref{eq:lifting_map_lemma}, this shows
\begin{subequations}\label{eq:proj_cond_measure_proof2}
\begin{align}
(\proj_{\mc{T}})_\sharp \nu_t^{\mc{S}_N}(\beta)
    &= (\exp_{\eins_{\mc{S}_N}}\circ Q)_\sharp \mc{N}(\cdot;t\lambda V_\beta, \Pi_0)\label{eq:proj_cond_measure_proof2_kickoff}\\
    &\overset{\eqref{eq:lifting_map_lemma}}{=} 
    (T\circ \exp_{\eins_{\mc{W}}})_\sharp \mc{N}(\cdot;t\lambda V_\beta, \Pi_0)\\
    &\overset{\eqref{eq:tangent_gaussian_path}}{=} 
    (T\circ \exp_{\eins_{\mc{W}}})_\sharp \mc{N}_{t, \beta}\\
    &\overset{\eqref{eq:nu_cond_lifted_gauss}}{=} 
    T_\sharp \nu_t(\beta)
\end{align}
\end{subequations}
which is the assertion \eqref{eq:proj_cond_measure}. 

Returning to \eqref{eq:proj_cond_measure_proof2_kickoff}, we compute the conditional vector field whose flow generates the path $(\proj_{\mc{T}})_\sharp \nu_t^{\mc{S}_N}(\beta)$ by
\begin{equation}\label{eq:proj_cond_vectorfield}
u_t^{\mc{T}}(q|\beta)
    = d\exp_{\eins_{\mc{S}_N}}(v)[\lambda QV_\beta] = R_q[\lambda QV_\beta]
\end{equation}
with $v = \exp_{\eins_{\mc{S}_N}}^{-1}(q)$, analogous to  \eqref{eq:condvectorfield-model}.  
This shows the shape of the flow matching criterion \eqref{eq:proj_flow_matching}. It remains to show that it is equal to \eqref{eq:af_flow_matching}.

Substituting the ansatz $\wt f_\theta = Q \circ F_\theta \circ M$ into this criterion gives
\begin{equation}\label{eq:af_flow_matching_metasimplex}
\mc{L}_{\mrm{RCFM}}^{\mc{T}} = 
\EE_{t\sim\rho, \beta\sim p, W\sim \nu_t(\beta)}\,\Big[
    \big\|R_{T(W)}[\lambda Q(V_\beta) - (Q\circ F_\theta)(W,t)]\big\|_{T(W)}^2\Big].
\end{equation}
By \cite[Theorem~3.1]{Boll:2024aa}, $T\colon \mc{W}_{c}\to\mc{T}\subseteq \mc{S}_N$ defined by \eqref{eq:def-T-embedding} is a Riemannian isometry. Thus, for any vector field $X\colon \mc{W}_{c}\to \mc{T}$ and any $W\in \mc{W}_{c}$, it holds that
\begin{equation}\label{eq:inner_prod_isometry}
\la R_W[X], R_W[X]\ra_W = \big\la dT_W\big[R_W[X]\big], dT_W\big[R_W[X]\big]\big\ra_{T(W)}.
\end{equation}
Furthermore, by \cite[Theorem~3.5]{Boll:2024aa}, one has
\begin{equation}\label{eq:T_diff_replicator}
dT_W\big[R_W[X]\big] = R_{T(W)}[QX].
\end{equation}
Taking \eqref{eq:inner_prod_isometry} and \eqref{eq:T_diff_replicator} together, \eqref{eq:af_flow_matching_metasimplex} transforms to
\begin{equation}\label{eq:W_fm_equivalence_proof}
\mc{L}_{\mrm{RCFM}}^{\mc{T}} =
\EE_{t\sim\rho, \beta\sim p, W\sim \nu_t(\beta)}\,\Big[
    \big\|R_{W}[\lambda V_\beta - F_\theta(W,t)]\big\|_{W}^2\Big]
\end{equation}
which is \eqref{eq:af_flow_matching}.
\end{proof}
\label{sec:appendix-proofs}
\section{Experiments: Details}\label{sec:appendix-experiments}
\subsection{Details of Class Scaling Experiment}\label{sec:class_scaling_app}
To parameterize $F_\theta$, we use the same convolutional architecture used in \cite{Stark:2024aa}.
We train for 500k steps of the Adam optimizer with constant learning rate $3\cdot 10^{-4}$ and batch size $128$.
We reproduce the Dirichlet flow matching results and linear flow matching baseline by using the code of \cite{Stark:2024aa}. The experiment shown in Figure~\ref{fig:stark_simplex_scaling} is slightly harder than the version in \cite{Stark:2024aa}, because we limit training to 64k steps at batch size 512 for Dirichlet- and linear flow matching.
Accordingly, both assignment flow methods are trained for 250k steps at batch size 128, such that around 32M data are seen by each model during training.

\subsection{Details of Generating Image Segmentations}\label{sec:experiments_img_app}

\subsubsection{Cityscapes Data Preparation}
Rather than the original $c=33$ classes, we only use the $c=8$ classes specified as \emph{categories} in \emph{torchvision}. The same subsampling of classes was used in the related work \cite{Hoogeboom:2021aa}. They additionally perform spatial subsampling to $32\times 64$. Instead, we subsample the spatial dimensions (\emph{NEAREST} interpolation) to $128\times 256$.

\subsubsection{Cityscapes Training}
For the Cityscapes experiment, we employ the UNet architecture of \cite{Dhariwal2021diffusion} with 
\emph{attention$\_$resolutions} (32, 16, 8), \emph{channel$\_$mult} (1,1,2,3,4), 4 attention heads, 3 blocks and 64 channels.
We trained for 250 epochs using Adam with cosine annealing learning rate scheduler starting at learning rate 0.0003 and batch size 4.
The distribution $\rho$ of times $t$ used during training is an exponential distribution with rate parameter $\lambda = 0.25$.
For sampling, we integrate up to $t_{\max} = 15$.

\subsubsection{Binarized MNIST Data Preparation}
We pad the original $28\times 28$ images with zeros to size $32\times 32$ to be compatible with spatial downsampling employed by the UNet architecture. 
Binarization is performed by pixelwise thresholding at grayvalue $0.5$.

\subsubsection{Binarized MNIST Training}
We modify the same architecture used for Cityscapes to \emph{attention$\_$resolutions} (16), \emph{channel$\_$mult} (1,2,2,2), 4 attention heads, 2 blocks and 32 channels. The same training regimen is used as for Cityscapes except for an increase in batch size to 256.
The distribution $\rho$ of times $t$ used during training is an exponential distribution with rate parameter $\lambda = 0.5$.
For sampling, we integrate up to $t_{\max} = 10$.
In table~\ref{tab:likelihood}, we use the same UNet architecture and training regimen for the comparison method \cite{Boll:2024ab} ($t\to 1$).

\section{Likelihood Computation: Details}\label{sec:appendix-likelihood}

Assume we have learned a probability path $\nu_t$ and a final pushfoward distribution $\nu_\infty$. 
In practice, numerical integration needs to be stopped after a finite time $t = t_{\max}$, reaching a numerical pushforward distribution $\nu_{t_{\max}}\approx \nu_\infty$. 
Drawing samples from $\wt{p} = \EE_{W\sim \nu_{t_{\max}}}[T(W)]$ is a two-stage process: $W\sim \nu_{t_{\max}}$ is drawn first, followed by sampling $\beta\sim T(W)$. Due to the numerical need to stop integration at finite time, $T(W)$ may in practice not have fully reached a discrete Dirac distribution. For long sequences of random variables, such as text or image modalities, this can lead to undesirable noise in the output samples. 
A way to combat this numerical problem is by rounding to a Dirac measure before sampling. This procedure can be interpreted within the framework of \emph{dequantization}, which we elaborate in Section~\ref{sec:dequantization}. 

In practice, $W\sim \nu_{t_{\max}}$ is typically close to a discrete Dirac already (cf.~Figure \ref{fig:T0-norms}), so rounding has little impact on the represented joint distribution.
Nevertheless, the rounding process is formally a different model than $\wt{p} = \EE_{W\sim \nu_{t_{\max}}}[T(W)]$, which we explicitly distinguish for the purpose of computing likelihoods. Recall the definition \eqref{eq:rounding_region} of subsets $r_\beta\subseteq \mc{W}$ with each $W\in r_\beta$ assigning the largest probability to the labels $\beta$. 
The points in $r_\beta$ are also the ones which round to $\ol{W}_\beta$\footnote{The sets $r_\beta$ technically overlap on the boundary, but all intersections have measure zero.}.
Thus, the labeling $\beta\in [c]^n$ has likelihood
\begin{equation}\label{eq:likelihood_rounding_model}
    \wt{p}^{r}_\beta = \EE_{W\sim \nu_{t_{\max}}}[1_{r_\beta}(W)] = \PP_{\nu_{t_{\max}}}(r_\beta)
\end{equation}
under the rounding model $\wt{p}^{r}$, with $1_{r_\beta}$ denoting the indicator function of $r_\beta$.
This is numerically similar to the likelihood under our original model
\begin{equation}\label{eq:joint_model_distribution}
    \wt{p}_\beta = \EE_{W\sim \nu_\infty}[T(W)_\beta]
\end{equation}
and matches it in the limit $t\to\infty$, provided that (almost) every trajectory $W(t)$ approaches an extreme point of $\ol{\mc{W}_{c}}$ under the learned assignment flow dynamics.

We will now devise an importance sampling scheme for efficient and numerically stable approximation of the integral in \eqref{eq:likelihood_rounding_model}, that analogously applies to \eqref{eq:joint_model_distribution}.
Let $\proposaldist$ be a proposal distribution with full support on $\mc{W}_{c}$ which has most of its mass concentrated around a point $q_\beta\in\mc{W}_{c}$ close to $\ol{W}_\beta$. Then
\begin{equation}\label{eq:priority_sampling}
    \PP_{\nu_{t_{\max}}}(r_\beta) = \EE_{W\sim \proposaldist}\Big[1_{r_\beta}(W)\frac{\nu_{t_{\max}}(W)}{\proposaldist(W)}\Big]
\end{equation}
where we assumed that both $\nu_{t_{\max}}$ and $\proposaldist$ have densities with respect to the Lebesgue measure and used again the symbols $\nu_{t_{\max}}$ and $\proposaldist$ to denote these densities. The rationale behind this construction is that, since we learned $\nu_{t_{\max}}$ to concentrate close to points $\ol{W}_\beta$, drawing most samples close to $q_\beta$ will reduce the estimator variance compared to sampling \eqref{eq:likelihood_rounding_model} directly.
In high dimensions, the quantities in \eqref{eq:priority_sampling} are prone to numerical underflow, which motivates the transformation
\begin{subequations}\label{eq:logsumexp_trick}
\begin{align}
\log \PP_{\nu_{t_{\max}}}(r_\beta) 
    &= \log \EE_{W\sim \proposaldist}\Big[1_{r_\beta}(W)\frac{\nu_{t_{\max}}(W)}{\proposaldist(W)}\Big]
    \\
    &= \log \EE_{W\sim \proposaldist} \big[\exp\big(\log 1_{r_\beta}(W) + \log \nu_{t_{\max}}(W) -\log \proposaldist(W)\big)\big].
\end{align}
\end{subequations}
After replacing the expectation with a mean over samples drawn from $\proposaldist$, we can evaluate \eqref{eq:logsumexp_trick} by leveraging stable numerical implementations of the $\text{logsumexp}$ function. 

For every evaluation of the integrand, we evaluate log-likelihood under $\proposaldist$ in closed form as well as log-likelihood under $\nu_{t_{\max}}$ through numerical integration backward in time, leveraging the instantaneous change of variables \eqref{eq:instant_change_of_variables} and Hutchinson's trace estimator \eqref{eq:Hutchinson}.
Note the conventions $\log 0 = -\infty$ and $\exp(-\infty) = 0$ employed in \eqref{eq:logsumexp_trick}. The analogous expression for \eqref{eq:likelihood_rounding_model} reads
\begin{equation}\label{eq:logsumexp_trick_Tansatz}
    \log \wt{p}_\beta = \log \EE_{W\sim \proposaldist} \big[\exp\big(\log T(W)_\beta + \log \nu_{t_{\max}}(W) -\log \proposaldist(W)\big)\big]
\end{equation}
and we can further expand
\begin{equation}\label{eq:logsumexp_trick_T}
    \log T(W)_\beta = \log \prod_{i\in [n]} W_{i,\beta_i} = \sum_{i\in [n]} \log W_{i,\beta_i}
\end{equation}
to avoid numerical underflow.

\newpage
\bibliographystyle{amsalpha}
\bibliography{AFGenerative_v2}

\end{document}